\documentclass{article}
\usepackage[a4paper,left=3cm,right=3cm,top=3cm,bottom=3cm]{geometry}

\usepackage[utf8]{inputenc} 
\usepackage[T1]{fontenc}    
\usepackage{hyperref}       
\usepackage{url}            
\usepackage{booktabs}       
\usepackage{amsfonts}       
\usepackage{amsmath}
\usepackage{amssymb}
\usepackage{amsthm}
\usepackage{graphicx}
\usepackage{nicefrac}       
\usepackage{microtype}      
\usepackage{algorithm}
\usepackage{algorithmic}
\usepackage{blkarray}
\usepackage{natbib}
\usepackage{authblk}
\usepackage{enumitem}
\usepackage{siunitx}
\usepackage{tikz}
\usetikzlibrary{decorations.pathreplacing}

\theoremstyle{definition}
\newtheorem{assume}{Assumption}
\newtheorem{cor}{Corollary}
\newtheorem{defn}{Definition}
\newtheorem{prop}{Proposition}

\newtheorem{lem}{Lemma}
\newtheorem{thm}{Theorem}

\newcommand{\E}{\mathbb{E}}
\newcommand{\fand}{\quad \mathrm{and} \quad}

\newcommand{\Ipq}{\mathbf{I}_{p,q}}
\newcommand{\m}[1]{\mathbf{#1}}

\newcommand{\Opq}{\mathbb{O}(p,q)}
\newcommand{\Prob}{\mathbb{P}}
\newcommand{\R}{\mathbb{R}}

\newcommand{\ud}{\, \mathrm{d}}

\newcommand{\Z}{\mathcal{Z}}

\DeclareMathOperator*{\argmin}{arg\,min}
\DeclareMathOperator{\diag}{diag}

\DeclareMathOperator{\rank}{rank}
\DeclareMathOperator{\var}{var}

\title{Spectral embedding of weighted graphs}
\author[*]{Ian Gallagher}
\author[*]{Andrew Jones}
\author[**]{Anna Bertiger}
\author[***]{Carey Priebe}
\author[*]{Patrick~Rubin-Delanchy}
\affil[*]{University of Bristol, U.K.}
\affil[**]{Microsoft, U.S.A.}
\affil[***]{Johns Hopkins University, U.S.A.}

\begin{document}

\maketitle



\bigskip
\begin{abstract}
When analyzing weighted networks using spectral embedding, a judicious transformation of the edge weights may produce better results. To formalize this idea, we consider the asymptotic behavior of spectral embedding for different edge-weight representations, under a generic low rank model. We measure the quality of different embeddings --- which can be on entirely different scales --- by how easy it is to distinguish communities, in an information-theoretic sense. For common types of weighted graphs, such as count networks or p-value networks, we find that transformations such as tempering or thresholding can be highly beneficial, both in theory and in practice. 

\end{abstract}

\noindent%

\section{Introduction}
\label{Sec:Intro}

Many real-world phenomena, from worldwide communications on the internet to microscopic protein-protein interactions, produce data representing pairwise interactions between entities, which can naturally be interpreted as the nodes and weighted edges of a graph. One approach to analysing these structures is through graph embedding, mapping the nodes of the graph into a low dimensional space, in a manner intended to preserve salient aspects of the original structure. These embeddings can be used for a range of tasks, for example, node clustering \citep{girvan2002community} and edge weight prediction \citep{bell2007modeling}. 

For networks with unweighted edges, there are many techniques to convert a graph into a low dimensional embedding of the nodes and the surveys \citet{Goyal_2018, cai2018comprehensive} provide a useful taxonomy of graph embedding algorithms. Matrix factorization based embeddings work by computing decompositions of either the graph adjacency matrix or matrices of node proximity derived from these, such as the normalized Laplacian or the Katz similarity matrix \citep{ou2016asymmetric}. More recently, deep learning algorithms have been applied to find graph embeddings, most notably, DeepWalk \citep{Perozzi_2014}, and one of its successors, node2vec \citep{grover2016node2vec}, which apply ideas from natural language processing to random walks generated on the network.

While deep learning techniques provide empirically useful embeddings for real-world networks, translating these graph algorithms into principled statistical procedures is not straightforward. Previous work has shown that the similar word2vec embedding can be thought of as a form of matrix factorization \citep{levy2014neural}. However, it was only very recently proved that the node2vec algorithm produces consistent community recovery under a stochastic block model \citep{zhang2021consistency}, the most basic model of a graph with community structure \citep{holland1983stochastic}.

Spectral embedding \citep{von2007tutorial} refers to the subset of these algorithms that use spectral decomposition for matrix factorization. A typical application of such graph embedding techniques is to partition the nodes into groups using a clustering algorithm, the combination of spectral embedding followed by clustering being known as spectral clustering. When using spectral clustering there are several options to consider:
\begin{itemize}
	\item Adjacency matrix representation. Matrices obtained through different regularization techniques, such as the normalized Laplacian matrix, can yield different clustering outputs \citep{sarkar2015role,tang2018limit,cape2019spectral}. In brain connectivity networks, clustering using the Laplacian embedding separates the left and right hemispheres of the brain, while adjacency embedding distinguishes between white and grey matter \citep{priebe2019two}.
	\item Eigenvector selection and scaling. Given the spectral decomposition of the matrix, nodes can be embedded selecting eigenvectors corresponding to largest, smallest, or largest-in-magnitude eigenvalues, and are optionally scaled by eigenvalue \citep{von2007tutorial,RCY11,rubin2021statistical}. 
	\item Clustering algorithm. Given the node embedding, there is a choice of clustering algorithm, for example, the benefits of $K$-means versus Gaussian mixture modelling \citep{tang2018limit,athreya2017statistical,rubin2021statistical} have been investigated in depth.
\end{itemize}
The relative simplicity of spectral embedding compared to other (e.g.\ deep learning) approaches allows for a statistical description of results, including uncertainty. Knowledge of the theoretical error distribution can inform decisions on the points above. Under a stochastic block model, the theory suggests that no representation is uniformly `best' \citep{tang2018limit}, recommends using large positive and negative eigenvalues \citep{RCY11} in tandem with Gaussian clustering \citep{tang2018limit,athreya2017statistical,rubin2021statistical}, which makes the question of scaling somewhat irrelevant, since the groupings obtained from fitting a Gaussian mixture model with varying shaped components are invariant \citep{rubin2021statistical}.


In many real-world networks, edges are not simply present or absent, but can take a range of numeric values. In this paper we particularly have in mind situations where the edges represent counts or p-values (although the model and asymptotic theory is more general). The former situation is ubiquitous in network monitoring applications; in cyber-security each count might represent the traffic between two computers, which could be measured in different ways (e.g. byte count, packet count), aggregated over time \citep{kent2016cyber}. Alternatively, in biological applications, the count might represent a number of physically distinct links between entities, such as the number of synaptic connections between two neurons \citep{eichler2017complete}. Networks of p-values are particularly prevalent in genetics and population health, arising in gene co-expression analyses \citep{stuart2003gene,zhang2005general,bakken2016comprehensive}, genome wide association studies \citep{buniello2019nhgri}, and systematic Mendelian randomization \citep{hemani2018mr}. Large graph databases mapping detected associations and causal effects between genes, traits, drugs, and more, across disparate scientific studies, are being built in several ongoing initiatives \citep{wishart2006drugbank,epigraphdb2020bioinformatics,Davis2022}.

Because there is nothing preventing spectral embedding from being applied to a non-binary matrix, it makes a natural candidate for embedding a weighted graph. Techniques relying on random walks, such as node2vec, require extra mechanisms to incorporate weights on edges. Moreover, one might expect the main elements of the theory of spectral embedding to translate to weighted networks, and hope to extract similar methodological guidance. 

However, when moving from unweighted to weighted graphs, there emerges a new consideration, which is how to represent edge weight. An entry-wise transformation of the weighted adjacency matrix will result in a different spectral embedding which can capture different properties of the network. In the applications alluded to above, it is common to threshold the data --- p-values as significant/non-significant, or count data as positive versus zero --- or take the log --- to amplify the small p-values, or temper the large counts. 

To investigate how different edge representations affect spectral embedding, an intrinsic measure of embedding quality is needed. Under a weighted stochastic block model \citep{xu2020optimal}, it is natural to aim for good community separation. We will show that nodes belonging to the same community form a multivariate Gaussian cluster (a central limit theorem) and, as a proxy for community separability, measure how well those Gaussian distributions can be distinguished, in an information-theoretic sense. 
Thus we are able to show situations where data transformations such as thresholding really do give better results.

The remainder of this article is structured as follows. Section~\ref{Sec:ExecSum} gives a minimal working summary of our results, including guidance on what weight transformations to consider in different applications. Section~\ref{Sec:Applications} validates these recommendations in two examples. The first explores how the choice of edge weight representation affects clustering in a synthetic network of p-values, while the second explores how tempering large counts can reveal real-world geometry in a network of air traffic data. In Section~\ref{Sec:WGRDPG} we define a generic low-rank model, an extension of the generalized random dot product graph \citep{rubin2021statistical} to include weighted edges, and describe the asymptotic behavior of spectral embedding under this model. Weighted standard and degree-corrected stochastic block models are special cases.  In Section~\ref{Sec:CI}, we define and motivate size-adjusted Chernoff information as a measure of community separation under the weighted stochastic block model and derive several implications for choosing between different edge weight representations. Section~\ref{Sec:Conc} concludes. All proofs are in the Appendix.


\section{Executive summary}
\label{Sec:ExecSum}
This article concerns the statistical analysis of a weighted graph on $n$ nodes, which we take to be represented by a symmetric matrix $\m{A} \in \R^{n \times n}$. Given a dimension $d \leq n$, the spectral embedding of this graph is a matrix $\m{X}_\m{A} \in \R^{n \times d}$ (see Definition~\ref{Def:ASE}); row $i$ of $\m{X}_\m{A}$ provides a $d$-dimensional representation of node~$i$.

Under a weighted stochastic block model, the nodes are grouped into $K$ communities and entries of $\m{A}$ corresponding to the same pair of communities are identically distributed (Definition~\ref{Def:WSBM}). 
The points of $\m{X}_\m{A}$ corresponding to a community asymptotically follow a $d$-dimensional Gaussian distribution, so that $\m{X}_\m{A}$ follows a $K$-component Gaussian mixture model. 
The separability between every pair of Gaussian distributions, and hence the corresponding communities, is measured using a variant of Chernoff information.

Our theory results in some guidelines for choosing between $\m{A}$ and other entry-wise transformations of $\m{A}$.
\begin{itemize}
    \item If possible, convert $\m{A}$ into a sparse matrix using an entry-wise affine transformation ($\m{A}_{ij} \rightarrow a \m{A}_{ij} + b$, $a \ne 0$) to enable faster spectral embedding computation (at no cost to statistical performance).
    \item If $\m{A}$ contains p-values, taking the logarithm ($\m{A}_{ij} \rightarrow -\log \m{A}_{ij}$), or thresholding as significant/non-significant ($\m{A}_{ij} \rightarrow \mathbb{I}(\m{A}_{ij} \le \tau)$, for some $\tau$) will usually be superior to spectral embedding $\m{A}$ directly. Choose taking the logarithm if the tests have low power, and thresholding if the network is sparse.
    \item If $\m{A}$ contains extreme values, such as large counts, apply some form of tempering, for example, using a fractional power ($\m{A}_{ij} \rightarrow \m{A}^{1/2}_{ij}$) or taking the logarithm of the positive edge weights $(\m{A}_{ij} \rightarrow \mathbb{I}(\m{A}_{ij} > 0) \log \m{A}_{ij}$). 
\end{itemize}

\section{Applications}
\label{Sec:Applications}
\subsection{Pairwise p-value data}
\label{Sec:p-values}
Consider the problem of detecting a cluster of anomalous activity on a network, on the basis of observed p-values $p_{ij}$ for pairs of nodes, quantifying our level of surprise in their activity. For example, in cyber-security applications a low p-value might occur if, relative to historical behavior, a much smaller or larger volume of communication was observed \citep{heard2010bayesian}, if a communication used an unusual channel \citep{nah16} or took place at a rare time of day \citep{PW18}. In genetics and population health, these p-values might instead capture pairwise significance-of-association measures between genes, traits, drugs, and more. 

\subsubsection{Model setup}
Consider a two-community weighted stochastic block model with edge weights modelled by
\begin{equation*}
    \m{A}_{ij} \mid \m{Z}_i, \m{Z}_j  \stackrel{\mathrm{ind}}{\sim} 
    (1 - \rho) \delta_1  + \rho \, \textrm{Beta}(\alpha_{\m{Z}_i \m{Z}_j}, 1),
\end{equation*}
where $\m{Z}_i \in \{1,2\}$ denotes the community of node $i$, $\delta_1$ is the delta distribution that places all its probability mass at one, $\alpha_{22} = 1$, and $\alpha_{11}, \alpha_{12} \le 1$ with at least one strict inequality. In this model, nodes in the first community ($\m{Z}_i = 1$) are regarded as `anomalous', and nodes in the second ($\m{Z}_i = 2$) `normal'. Accordingly, p-values observed between two normal nodes are uniformly distributed, whereas p-values for interactions involving anomalous nodes are stochastically smaller than uniform. It would often be unrealistic to expect to observe a p-value for every pair, so we assume each p-value is only observed with probability $\rho$, independently of the communities involved. An unobserved p-value is encoded as the value 1. In a cyber-security setting this could represent no data being transferred between computers, or an initial triage suggesting it is not worth running a full test for anomalous behavior. It might also be unrealistic to assume the p-values are conditionally independent; in genome wide association studies the actual dependency is quite complex \citep{li2012evaluating}. In such situations, we would recommend treating clustering results as potentially indicative, rather than scientifically conclusive.

\subsubsection{Edge weight representation}
We discuss three different representations of p-values in this network. The first is the affine entry-wise transformation $\m{A}_{ij}^\textrm{P} = 1 - \m{A}_{ij}$. This has the practical advantage of assigning uninformative p-values in the network to zero, making it easier to exploit sparsity in the adjacency spectral embedding computation. Lemma~\ref{Thm:Affine_SBM} provides an asymptotic sense in which this transformation will not affect the quality of the embedding for distinguishing communities. Entries $\m{A}_{ij}^\textrm{P}$ follow a zero-inflated beta distribution.

Based on Fisher's method for combining p-values \citep{fisher1934statistical}, one could consider embedding the matrix of log p-values $\m{A}_{ij}^\textrm{L} = -\log (\m{A}_{ij})$. Ignoring any graph setting, this is the uniformly most powerful method of combining p-values if they have a $\text{Beta}(\alpha, 1)$ distribution with $\alpha < 1$ under the alternative hypothesis \citep{heard2018choosing}. Entries $\m{A}_{ij}^\textrm{L}$ follow a zero-inflated exponential distribution.

Finally, one may only be interested in p-values less than a chosen threshold $\tau \in (0,1)$ \citep{heard2010bayesian} and consider embedding the matrix with entries $\m{A}_{ij}^\textrm{T} = \mathbb{I}(\m{A}_{ij} \le \tau)$. This converts the network into a binary graph and entries $\m{A}_{ij}^\textrm{T}$ follow a Bernoulli distribution 
dependent on the edge probability parameter $\rho$.

Table~\ref{Tab:p-values_rep} gives a summary of the three p-value transformations outlined above and the resulting edge weight distributions for the corresponding weighted stochastic block models.
\begin{table}[ht]
  \centering
  \begin{tabular}{@{}lll@{}}
  	\toprule
    Data representation & Matrix entries & Edge weight distribution \\
    \midrule
    $1-$p-values & $\m{A}^\textrm{P}_{ij} = 1 - \m{A}_{ij}$ & $(1 - \rho) \delta_0  + \rho \, \textrm{Beta}(\alpha, 1)$ \\
    log p-values & $\m{A}_{ij}^\textrm{L} = -\log (\m{A}_{ij})$ & $(1 - \rho) \delta_0  + \rho \, \textrm{Exp}(\alpha)$ \\
    Threshold p-values & $\m{A}_{ij}^\textrm{T} = \mathbb{I}(\m{A}_{ij} \le \tau)$ & $\textrm{Bernoulli}(\rho \tau^\alpha)$ \\
    \bottomrule
  \end{tabular}
  \caption{Three different representations of the p-values network data and the corresponding edge weight distributions.}
  \label{Tab:p-values_rep}
\end{table}

\subsubsection{Example embeddings}

To show the different embeddings, we generate a weighted stochastic block model with $n = 1000$ nodes where a node exhibits anomalous behavior with probability $\pi_1 = 0.1$. We pick $\alpha_{11} = 0.1$ and $\alpha_{12} = 1.0$ and edge probability $\rho = 0.25$, so only the edge weights between anomalous nodes ($\m{Z}_i = \m{Z}_j = 1$) contain useful signal.

Figure~\ref{Fig:ASE_p-values_2SBM_distbn} shows the spectral embeddings into $d = 2$ dimensions obtained for three transformations of the p-values matrix $\m{A}$; the $1-$p-values $\m{A}^\textrm{P}_{ij}$, the log p-values $\m{A}^\textrm{L}_{ij}$ and the threshold p-values $\m{A}^\textrm{T}_{ij}$ with threshold $\tau = 0.1$. Points are colored according to their community assignment $\m{Z}_i$ with anomalous nodes being represented by red triangles.
\begin{figure}[htbp]
	\centering
	\includegraphics[width=\textwidth]{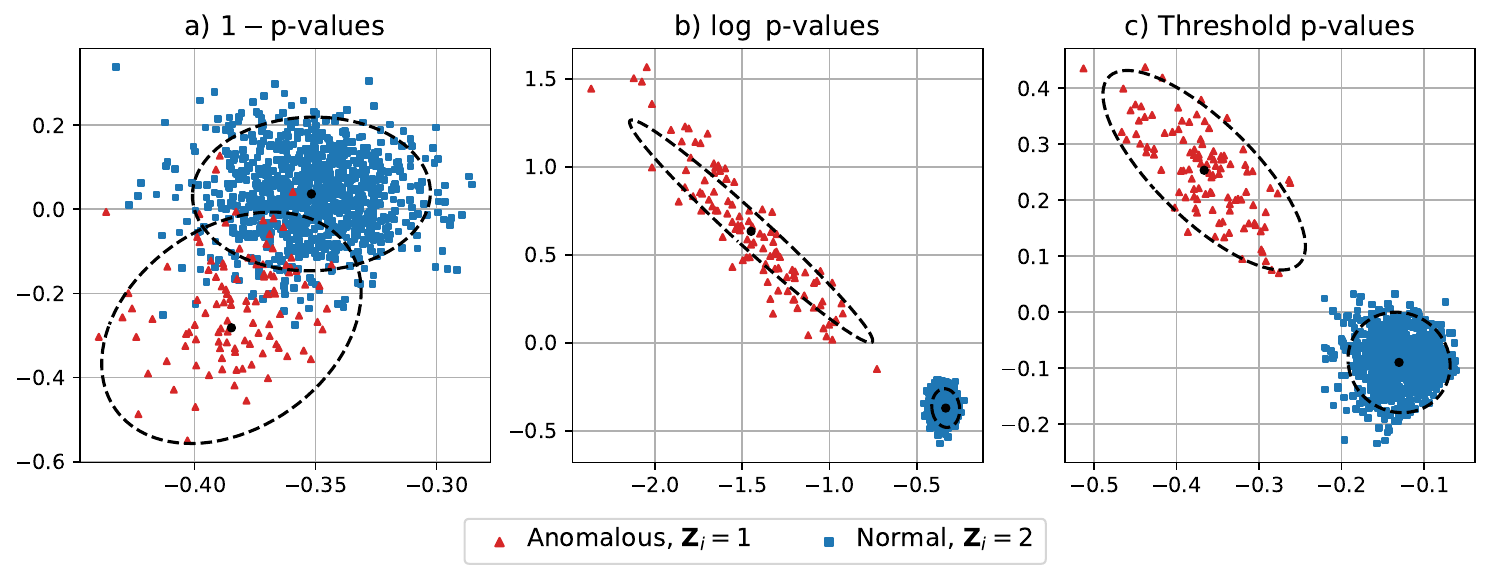}
	\caption{Two-dimensional embeddings of graphs following two-community weighted stochastic block models corresponding to a) the $1-$p-values matrix $\m{A}^\textrm{P}$, b) the log p-values matrix $\m{A}^\textrm{L}$ and c) the threshold p-value matrix $\m{A}^\textrm{T}$. The points are colored according to their true community assignment $\m{Z}_i$. Ellipses show the 95\% contours of the asymptotic Gaussian components calculated using \textbf{Corollary~\ref{Cor:WSBM_CLT}}.}
	\label{Fig:ASE_p-values_2SBM_distbn}
\end{figure}

In Figure~\ref{Fig:ASE_p-values_2SBM_distbn}a the embedding for the $1-$p-values matrix can partially distinguish between the anomalous and normal nodes in the network. However, in Figures~\ref{Fig:ASE_p-values_2SBM_distbn}b and \ref{Fig:ASE_p-values_2SBM_distbn}c, the embeddings for the log p-values matrix and threshold p-values matrix both separate the two types of nodes perfectly.


\subsubsection{Embedding comparison}
Size-adjusted Chernoff information introduced in Definition~\ref{Def:SACI} measures the separability of communities in the adjacency spectral embedding of a weighted stochastic block model for different representations of the edge weights in a network. For the three representations shown in Figure~\ref{Fig:ASE_p-values_2SBM_distbn}, the size-adjusted Chernoff information using the $1-$p-values matrix is $C^\textrm{P} = 1.2 \times 10^{-3}$, while for the log p-values matrix and threshold p-values matrix, the size-adjusted Chernoff information is $C^\textrm{L} = 4.8 \times 10^{-3}$ and $C^\textrm{T} = 4.9 \times 10^{-3}$ respectively, both a major improvement. 

We perform an experiment to compare the embeddings using $1-$p-values $\m{A}^\textrm{P}_{ij}$, log p-values $\m{A}^\textrm{L}_{ij}$ and threshold p-values $\m{A}^\textrm{T}_{ij}$ for three thresholds $\tau \in \{ 0.01, 0.05, 0.10 \}$. We consider versions of the model where the proportion of anomalous nodes is unbalanced or balanced, $\pi_1 \in \{0.1, 0.5\}$, and three edge probabilities, $\rho \in \{0.05, 0.25, 0.50\}$. For every pair of options, we compute the size-adjusted Chernoff information for each p-value representation for a range of beta distribution parameters $\alpha_{11}, \alpha_{12} \in (0,1]$.

Figure~\ref{Fig:ASE_p-values_phase} contains phase diagrams for each model showing which p-value transformation results in the greatest size-adjusted Chernoff information over a range of different parameters.
\begin{figure}[htbp]
	\centering
	\includegraphics[width=\textwidth]{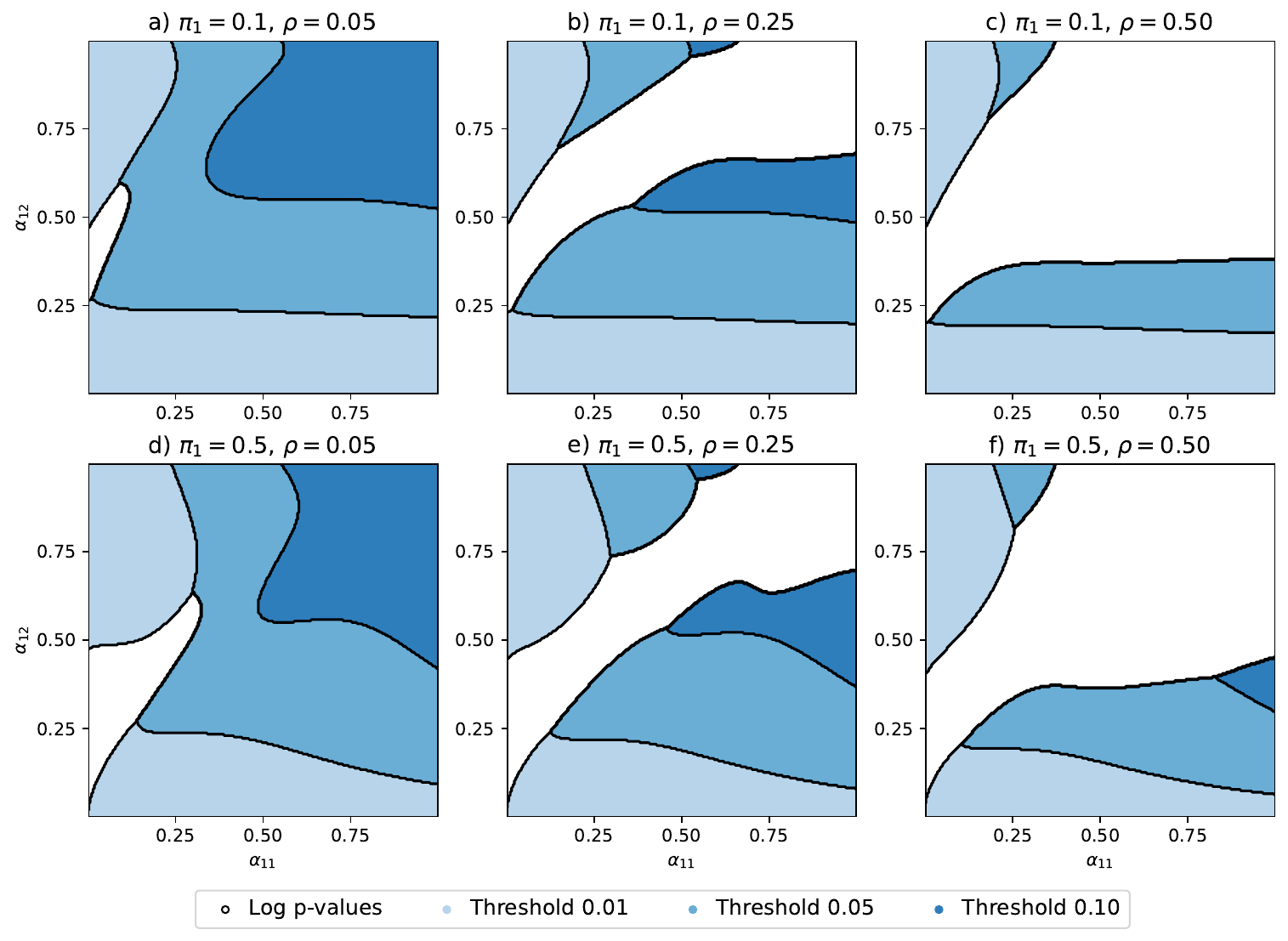}
	\caption{Phase diagrams showing which p-value transformation results in the greatest size-adjusted Chernoff information for models with anomaly probability $\pi_1 \in \{0.1, 0.5\}$ and edge probability $\rho \in \{0.05, 0.25, 0.50\}$. In the white regions, embedding log p-values achieves the greatest size-adjusted Chernoff information. In the colored regions, embedding threshold p-values is preferred, where the color indicates the threshold $\tau \in \{ 0.01, 0.05, 0.10 \}$ achieving maximal size-adjusted Chernoff information.}
	\label{Fig:ASE_p-values_phase}
\end{figure}
Most importantly, embedding the $1-$p-values (or p-values) matrix directly never results in the greatest size-adjusted Chernoff information. It is always preferable to use either the log p-values or threshold p-values. 

Choosing between the latter two and what threshold $\tau$ to use is trickier, although often there is not much difference and, in practice, it is worth trying both approaches. As a rule-of-thumb, using log p-values is better when the graph is dense ($\rho \geq 0.50$) and using threshold p-values better when the graph is sparse ($\rho \leq 0.05$). When the edge weights involving anomalous nodes do not contain much signal ($\alpha_{11} \approx 1$ and $\alpha_{12} \approx 1$) it is preferable to use log p-values or threshold p-values with a larger threshold, $\tau = 0.10$. When edge weights have strong signal ($\alpha_{11} \approx 0$ or $\alpha_{12} \approx 0$) it is preferable to use threshold p-values with a smaller threshold, $\tau = 0.01$.



\subsection{Air traffic data}
\label{Sec:Airport}
We analyze the crowdsourced OpenSky dataset detailing every flight between airports during 2019 \citep{strohmeier2021crowdsourced}. We construct a network of $n \approx 15,000$ nodes representing the different airports and create a weighted adjacency matrix $\m{A}$ where $\m{A}_{ij}$ is the total number of flights in either direction between two corresponding airports. This results in 1.4 million network edges with a heavy-tailed distribution, the largest edge weight representing over 56,000 flights between two busy airports. 

\subsubsection{Edge weight representation}
To temper the effect of large edge weights in the network, the entries of $\m{A}$ are transformed using the mapping $\m{A}_{ij} \to \m{A}^\gamma_{ij}$ for $\gamma \in [0,1]$. We analyze three different versions of the flight network; the original adjacency matrix $\m{A}$, the matrix with square root entries $\m{A}^{1/2}_{ij}$, and the binary adjacency matrix $\mathbb{I}(\m{A}_{ij} > 0)$ which states if there was any flight between two airports. The latter corresponds to the case $\gamma \to 0^+$ since then $\m{A}^\gamma_{ij} \to \mathbb{I}(\m{A}_{ij} > 0)$.

The transformed count data matrices are embedded into $d = 4$ dimensions selected using profile likelihood  \citep{zhu2006automatic}. Each embedding forms a collection of rays emanating from the origin, the geometry predicted by a weighted degree-corrected stochastic block model (Definition~\ref{Def:WDCSBM}). Each ray roughly corresponds to a continent, with the busier airports further from the origin. Following recommendations regarding community detection under a degree-corrected stochastic block model \citep{passino2022spectral}, the embedding is projected onto the unit sphere in $d = 3$ dimensions. Figure~\ref{Fig:Flight_ASE_Spherical} shows the first two dimensions of the spherical projection of the three different adjacency spectral embeddings.
\begin{figure}[htbp]
	\centering
	\includegraphics[width=\textwidth]{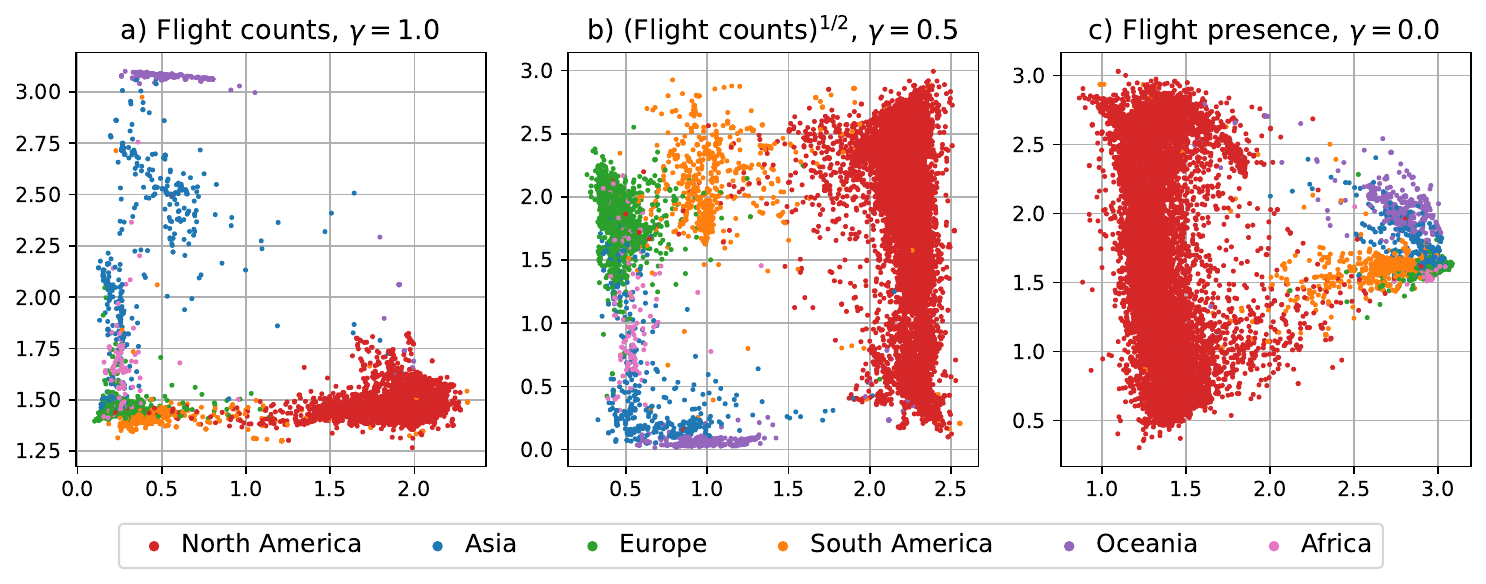}
	\caption{The first two dimensions of the spherical projection of the embeddings for the flight data network of a) the flight counts matrix $\m{A}_{ij}$, b) the square root flight counts matrix $\m{A}^{1/2}_{ij}$ and c) the flight presence event matrix $\mathbb{I}(\m{A}_{ij} > 0)$. Points are colored based on the continent of the airport.}
	\label{Fig:Flight_ASE_Spherical}
\end{figure}


The embedding using the square root of the flight counts in Figure~\ref{Fig:Flight_ASE_Spherical}b seems to form a ring capturing some aspect of the topology of the globe. This topology is less obvious in Figures~\ref{Fig:Flight_ASE_Spherical}a and \ref{Fig:Flight_ASE_Spherical}c.


\subsubsection{Embedding comparison}
To give a more quantitative argument for choosing $\gamma = 1/2$ in this example, we model the spherically projected points corresponding to each continent with multivariate Gaussian distributions and compute size-adjusted Chernoff information according to Definition~\ref{Def:SACI}, which we treat as an empirical measure of separation.


The size-adjusted Chernoff information for the flight counts matrix is $C = 0.299$ corresponding to the difficulty in distinguishing the continents of Europe and South America in Figure~\ref{Fig:Flight_ASE_Spherical}a. This improves to $C = 0.488$ for the square root flight counts matrix corresponding to Africa and Asia in Figure~\ref{Fig:Flight_ASE_Spherical}b, but degrades to $C = 0.278$ for the flight presence event matrix for the same continents in Figure~\ref{Fig:Flight_ASE_Spherical}c. 

\section{Weighted generalized random dot product graph}
\label{Sec:WGRDPG}

\begin{defn}[Weighted generalized random dot product graph]
\label{Def:WGRDPG}
Let $\Z$ be a sample space with probability distribution $F$. Let $\{H(\m{z_1}, \m{z_2}): \m{z}_1, \m{z}_2 \in \Z\}$ be a family of real-valued distributions with the following properties:
\begin{enumerate}[label=\emph{\textbf{\roman*.}}, labelindent=0pt, labelwidth=!]
\item (Symmetry) $H(\m{z}_1, \m{z}_2) = H(\m{z}_2, \m{z}_1)$ for all $\m{z}_1, \m{z}_2 \in \Z$.
\item (Low-rank expectation) There exists a map $\phi: \Z \to \R^d$ such that, for all $\m{z}_1, \m{z}_2 \in \Z$, if $A \sim H(\m{z}_1, \m{z}_2)$, then
\begin{equation*}
	\E(A) = \phi(\m{z}_1)^\top \Ipq \phi(\m{z}_2),
      \end{equation*}
      where  $\Ipq = \diag(1,\ldots,1,-1,\ldots,-1)$ is a diagonal matrix with $p$ ones followed by $q$ minus ones, with $p+q = d$.
\end{enumerate}
A symmetric matrix $\m{A} \in \R^{n \times n}$ is distributed as a \emph{weighted generalized random dot product graph} with signature $(p,q)$, if  $\smash{\m{Z}_1, \ldots, \m{Z}_n \stackrel{\mathrm{iid}}{\sim} F}$ and, for $i < j$,
\begin{equation*}
	\m{A}_{ij} \mid \m{Z}_i, \m{Z}_j  \stackrel{\mathrm{ind}}{\sim} H(\m{Z}_i, \m{Z}_j).
\end{equation*} 
\end{defn}

Letting  $\m{X}_i = \phi(\m{Z}_i) \in \R^d$, we define the point cloud $\smash{\m{X} = [\m{X}_1 | \ldots | \m{X}_n]^\top \in \R^{n \times d}}$, so that, conditional on $\m{Z}_1, \ldots, \m{Z}_n$
\begin{equation*}
  \E(\m{A}_{ij}) = \m{P}_{ij} \fand \var(\m{A}_{ij}) = v(\m{Z}_i, \m{Z}_j),
\end{equation*}
where $\m{P} = \m{X} \Ipq \m{X}^\top \in \R^{n \times n}$ and $v(\m{z}_1, \m{z}_2)$ is the variance of $H(\m{z}_1, \m{z}_2)$.




\begin{defn}[Adjacency spectral embedding]
\label{Def:ASE}
Given the eigendecompositions
\begin{equation*}
	\m{A} = \m{U}_\m{A} \m{\Lambda}_\m{A} \m{U}_\m{A}^\top + \m{U}_{\m{A},\perp} 	\m{\Lambda}_{\m{A},\perp} \m{U}_{\m{A},\perp}^\top \fand
	\m{P} = \m{U}_\m{P} \m{\Lambda}_\m{P} \m{U}_\m{P}^\top,
\end{equation*}
where $\m{U}_\m{A}, \m{U}_\m{P} \in \mathbb{O}(n \times d)$ and $\m{\Lambda}_\m{A}, \m{\Lambda}_\m{P} \in \R^{d \times d}$ are diagonal matrices comprising the $d$ largest eigenvalues in absolute value of $\m{A}$ and $\m{P}$ respectively, denote by $\m{X}_\m{A}, \m{X}_\m{P} \in \R^{n \times d}$ the \emph{spectral embeddings}
\begin{equation*}
	\m{X}_\m{A} = \m{U}_\m{A} |\m{\Lambda}_\m{A}|^{1/2} \fand
	\m{X}_\m{P} = \m{U}_\m{P} |\m{\Lambda}_\m{P}|^{1/2},
\end{equation*}
where $|\cdot|$ applied to a diagonal matrix denotes the entry-wise absolute value. 
\end{defn}
Spectral embedding only allows estimation of $\m{X}_i$, rather than $\m{Z}_i$, and this has practical consequences, for example, when $\m{Z}_i \neq \m{Z}_j$ but $\m{X}_i = \m{X}_j$ under the weighted stochastic block model. 

\subsection{The weighted stochastic block model}
In the following definitions, $\Z = \{1, \ldots, K \}$ is a sample space with probability distribution $F$ and $\{H(\m{z}_1, \m{z}_2): \m{z}_1, \m{z}_2 \in \Z\}$ a symmetric family of real-valued distributions. The weighted stochastic block model is the most basic model for community structure in a weighted network.
\begin{defn}[Weighted stochastic block model]
    \label{Def:WSBM}
    A symmetric matrix $\m{A} \in \R^{n \times n}$ is distributed as a \emph{weighted stochastic block model}, if $\smash{\m{Z}_1, \ldots, \m{Z}_n \stackrel{\mathrm{iid}}{\sim} F}$ and, for $i < j$,
    \begin{equation*}
	   \m{A}_{ij} \mid \m{Z}_i, \m{Z}_j \stackrel{\mathrm{ind}}{\sim} H(\m{Z}_i, \m{Z}_j).
    \end{equation*}
\end{defn}
Given $\m{Z} \sim F$, we denote the probability a node is assigned to community $k$ by $\pi_k = \Prob(\m{Z} = k)$. Let $\m{B}\in \R^{K \times K}$ and $\m{C} \in \R_+^{K \times K}$ respectively denote the \emph{block mean} and \emph{block variance} matrices, that is, $\m{B}_{k \ell}$ and $\m{C}_{k \ell}$ are the mean and variance of $H(k, \ell)$, if these moments are defined.

To show the weighted stochastic block model is a weighted generalized random dot product graph, the following canonical construction for $\phi$ is proposed. Consider the eigendecomposition $\smash{\m{B} = \m{U}_\m{B} \m{\Lambda}_\m{B} \m{U}_\m{B}^\top}$, where $\m{\Lambda}_\m{B} \in \R^{d \times d}$ is a diagonal matrix comprising the $d$ non-zero eigenvalues ($p$ positive, $q$ negative) of $\m{B}$ and $\m{U}_\m{B} \in \mathbb{O}(K \times d)$ a matrix comprising its $d$ corresponding orthonormal eigenvectors. Constructing the spectral embedding $\smash{\m{X}_\m{B} = \m{U}_\m{B} | \m{\Lambda}_\m{B} |^{1/2} \in \R^{K \times d}}$, we define the map  $\phi: \Z \to \R^d$ to take an index $k \in\Z$, representing a community, to the corresponding row of the spectral decomposition, that is, $\phi(k) = (\m{X}_{\m{B}})_k$ and
\[\m{X}_i = \phi(\m{Z}_i) = {(\m{X}_{\m{B}})}_{\m{Z}_i}.\]
Since $\smash{\m{B} = \m{X}_\m{B} \Ipq \m{X}_\m{B}^\top}$, the low-rank expectation condition of Definition~\ref{Def:WGRDPG} is satisfied with dimension $d$ and signature $(p,q)$.


The following is a simple extension which allows us to model nodes within the same community to have different degrees ($\sum_i \mathbb{I}(A_{ij} > 0)$), inspired by the standard degree-corrected stochastic block model \citep{karrer2011stochastic}.


\begin{defn}[Weighted degree-corrected stochastic block model]
\label{Def:WDCSBM}
Let $G_k$, $k \in \Z$ be a collection of probability distributions on the interval $(0,1]$. A symmetric matrix $\m{A} \in \R^{n \times n}$ is distributed as a \emph{weighted degree-corrected stochastic block model}, if $\m{Z}_1, \ldots, \m{Z}_n \stackrel{\mathrm{iid}}{\sim} F$ and, for $i < j$,
\begin{equation*}
	\m{A}_{ij} \mid w_i, w_j, \m{Z}_i, \m{Z}_j \stackrel{\mathrm{ind}}{\sim} (1 - w_i w_j) \delta_0 + w_i w_j H(\m{Z}_i, \m{Z}_j),
\end{equation*}
where $w_i \mid \m{Z}_i \stackrel{\mathrm{ind}}{\sim} G_{\m{Z}_i}$ is a node-specific weight parameter.
\end{defn}
To represent this model as a weighted generalized random dot product graph, we assign to each node $i$ the latent position
\[\m{X}_i = w_i \: {(\m{X}_{\m{B}})}_{\m{Z}_i},\]
so that $\phi:  (0,1] \times \Z \rightarrow \R^d$. Thus, the latent positions $\m{X}_i$ live on $K$ rays emanating from the origin.

\subsection{Asymptotics}
\label{Sec:WGRDPG_Asym}
In this section we present results concerning the asymptotic behavior of the spectral embedding $\m{X}_\m{A}$, when taken as an estimate of the point cloud $\m{X}$.

We note at this junction that there is a degree of freedom in the choice of the map $\phi$, since the map $\smash{\phi'(\m{z}) = \m{Q}^\top \phi(\m{z})}$ will also satisfy the low-rank expectation assumption of Definition~\ref{Def:WGRDPG}, for any matrix $\smash{\m{Q} \in \Opq = \{\m{M} \in \R^{d \times d}\ :\ \m{M} \Ipq \m{M}^\top = \Ipq\}}$, the group of indefinite orthogonal transformations. Without imposing a canonical choice for $\phi$, one can only hope for the point cloud  $\m{X}_\m{A}$ to approach $\m{X}$ up to such a transformation. Statements of convergence, below, will therefore be made about the realigned point cloud $\m{X}_\m{A} \m{Q}_n$, for some $\m{Q}_n \in \Opq$ known only to an oracle who has access to $\m{A}$ and $\m{X}$. To achieve these asymptotic results, the following regularity conditions are introduced.



\begin{assume}
\label{Ass:Reg}
The weighted generalized random dot product graph satisfies the following conditions:
\begin{enumerate}[label=\emph{\textbf{\roman*.}}, labelindent=0pt, labelwidth=!]
\item (Minimal dimensionality) For $\m{Z} \sim F$, the random vector $\m{X} = \phi(\m{Z})$ has second moment matrix $\m{\Delta} = \E(\m{X} \m{X}^\top)$ with full rank $d$.
\item (Bounded expectation) There exist universal constants $a,b \in \R$ such that, for all $\m{z}_1, \m{z}_2 \in \Z$, if $A \sim H(\m{z}_1, \m{z}_2)$, then
\begin{equation*}
	\E(A) = \phi(\m{z}_1)^\top \Ipq \phi(\m{z}_2) \in [a,b].
\end{equation*}
\item (Exponential tails) There exists universal constants $\alpha > 0$ and $\beta_{\rho} > 0$ for each $\rho \in \R$, such that for all $\m{z}_1, \m{z}_2 \in \Z$, if $A \sim H(\m{z}_1, \m{z}_2)$, then
\begin{equation*}
	\Prob \left\{|A| < \beta_\rho \log^\alpha(t) \right\} > 1 - t^{-\rho}.
\end{equation*}
\end{enumerate}

\end{assume}
Condition \emph{\textbf{i.}} is no real constraint, insisting only that the map $\phi$ is chosen sensibly; if the assumption fails, it is possible to find a map into $\R^{d'}$ with $d' < d$, where the relevant second moment matrix has full rank. 

Let $\left\| \cdot \right\|_{2 \to \infty}$ denote the two-to-infinity norm of a matrix \citep{two_to_infinity}, that is, the maximum row-wise Euclidean norm.

\begin{thm}[Two-to-infinity norm]
\label{Thm:2Inf}
Let $\m{A}$ be an instance of a weighted generalized random dot product graph. There exists a sequence of matrices $\m{Q}_n \in \Opq$ such that
\begin{equation*}
	\left\| \m{X}_\m{A} \m{Q}_n - \m{X} \right\|_{2 \to \infty}  = 
	\mathrm{O} \left(n^{-1/2} \log^{3\alpha + 3/2}(n) \right).
\end{equation*}
\end{thm}
Let $\Phi \{ \m{\mu}, \m{\Sigma}\}$ denote the multivariate normal cumulative distribution function with mean $\mu$ and covariance matrix $\m{\Sigma}$.
\begin{thm}[Central limit theorem]
\label{Thm:CLT}
Let $\m{A}$ be an instance of a weighted generalized random dot product graph. Given $\m{z} \in \Z$, define the covariance-valued function
\begin{equation*}
  \m{\Sigma}(\m{z}) = \Ipq \m{\Delta}^{-1} \E \left\{ v(\m{z}, \m{Z}) \phi(\m{Z}) \phi(\m{Z})^\top \right\} \m{\Delta}^{-1} \Ipq,
\end{equation*}
where $\m{Z} \sim F$. Then, there exists a sequence of matrices $\m{Q}_n \in \Opq$ such that, for all $\m{x} \in \R^d$, 
\begin{equation*}
	\Prob \left\{ n^{1/2} \left( \m{X}_\m{A} \m{Q}_n - \m{X} \right)_i^\top \leq \m{x}
	~|~ \m{Z}_i = \m{z} \right\} \to \Phi \left\{ \m{x}, \m{\Sigma}(\m{z}) \right\}.
\end{equation*}
\end{thm}
An oracle would construct the matrix $\m{Q}_n$ in two stages; first using a modified Procrustes-style procedure to align $\m{X}_\m{A}$ with $\m{X}_\m{P}$, and then applying a second transformation to align $\m{X}_\m{P}$ with $\m{X}$.  For the first step, let $\m{U}_\m{P}^\top\m{U}_\m{A} + \Ipq\m{U}_\m{P}^\top\m{U}_\m{A}\Ipq$ admit the singular value decomposition
\begin{equation*}
\m{U}_\m{P}^\top\m{U}_\m{A} + \Ipq\m{U}_\m{P}^\top\m{U}_\m{A}\Ipq = \m{W}_1\m{\Sigma}\m{W}_2^\top,
\end{equation*}
and let $\m{W} = \m{W}_1\m{W}_2^\top$. The matrix $\m{W}$ solves the one mode \emph{orthogonal} Procrustes problem
\begin{equation*}
\m{W} = \argmin_{\m{Q} \in \mathbb{O}(d)} \|\m{U}_\m{A} - \m{U}_\m{P}\m{Q}\|_F^2 + \|\m{U}_\m{A}\Ipq - \m{U}_\m{P}\Ipq\m{Q}\|_F^2,
\end{equation*}
where $\m{W}$ in $\mathbb{O}(d) \cap \mathbb{O}(p,q)$ since $\m{U}_\m{P}^\top\m{U}_\m{A} + \Ipq\m{U}_\m{P}^\top\m{U}_\m{A}\Ipq$ is a block-diagonal matrix. We use $\m{W}^\top$ to align $\m{X}_\m{A}$ with $\m{X}_\m{P}$, that is, $\m{X}_\m{A} \m{W}^\top \approx \m{X}_\m{P}$.

For the second step, observe that since $\m{X}_\m{P}\Ipq\m{X}_\m{P}^\top = \m{P} = \m{X}\Ipq\m{X}^\top$, there exists some matrix $\m{Q}_\m{X} \in \mathbb{O}(p,q)$ such that $\m{X}_\m{P} = \m{X}\m{Q}_\m{X}$, so that $\m{X}_\m{A} \m{W}^\top \m{Q}_\m{X}^{-1} \approx \m{X}$. The matrix $\m{Q}_{n}$ is then given by $\m{Q}_{n} = \m{W}^\top \m{Q}_\m{X}^{-1}$.



Since the weighted stochastic block model is a special case of the weighted generalized random dot product graph, Theorem~\ref{Thm:CLT} can be used to describe the asymptotic behavior of the corresponding embedding.

\begin{cor}[Weighted stochastic block model central limit theorem]
\label{Cor:WSBM_CLT}
Let $\m{A}$ be an instance of a weighted stochastic block model satisfying Assumption~\ref{Ass:Reg}. Given $k \in \{1, \ldots, K \}$ and $\m{Z} \sim F$, let
\begin{equation*}
	\m{\Sigma}_k = \Ipq \m{\Delta}^{-1} \left\{ \sum_{\ell=1}^{K} \pi_\ell \m{C}_{k \ell} (\m{X}_{\m{B}})_\ell (\m{X}_{\m{B}})_\ell^\top \right\} \m{\Delta}^{-1} \Ipq,
\end{equation*} 
then Theorem~\ref{Thm:CLT} implies that, for all $\m{x} \in \R^d$, 
\begin{equation*}
	\Prob \left\{ n^{1/2} \left( \m{X}_\m{A} \m{Q}_n - \m{X} \right)_i^\top \leq \m{x}
	~|~ \m{Z}_i = k \right\} \to \Phi \left( \m{x}, \m{\Sigma}_k \right).
\end{equation*}
\end{cor}

Of the regularity conditions specified in Assumption~\ref{Ass:Reg}, condition \emph{\textbf{i.}} is satisfied by our construction of $\m{X}_\m{B}$ and choice of $d = \rank(\m{B})$. The remaining conditions depend on the family of distributions $\{H(\m{z}_1, \m{z}_2): \m{z}_1, \m{z}_2 \in \Z\}$. Condition \emph{\textbf{ii.}} is satisfied if the means (the entries of $\m{B}$) are finite, while condition \emph{\textbf{iii.}} holds for many common distributions considered in this article, such as the Gaussian and Beta distributions.

The implication of Corollary~\ref{Cor:WSBM_CLT} is that the adjacency spectral embedding of a graph following a weighted stochastic block model asymptotically produces a point cloud that is a linear transformation of independent, identically distributed draws from a Gaussian mixture model with $K$ components. The community structure (partition of the nodes) found by Gaussian clustering is invariant to a full-rank, linear transformation of the points, so that unidentifiability of $\smash{\m{Q}_n^{-1} \in \Opq}$ is immaterial for community recovery \citep{rubin2021statistical}. The $K$-means clustering algorithm \citep{steinhaus1956division} is less appropriate, since it implicitly assumes spherical covariance matrices and is not invariant to the action of $\m{Q}_n^{-1}$.

We conclude this section with an example which counters the common notion that spectral embedding only uncovers community structure in the mean matrix $\m{P} = \E(\m{A} \mid \m{Z}_1, \ldots, \m{Z}_n)$ \citep[Section 4]{xu2020optimal}.

\subsection{Gaussian distributions with equal means}
\label{sec:gaussian_sbm}


Consider a weighted stochastic block model with two communities where matrix entries are Gaussian random variables with symmetric family of distributions $\{H(\m{z}_1, \m{z}_2): \m{z}_1, \m{z}_2 \in \Z\}$, where
\begin{equation*}
    H(\m{z}_1, \m{z}_2) = \textrm{Normal}(1, \sigma^2_{\m{z}_1 \m{z}_2}).
\end{equation*}
The mean of each distribution is the same regardless of the community assignment, only the variance changes. The block mean matrix $\m{B}$ is equal to the all-one matrix and $\rank(\m{B}) = 1$ meaning the function $\phi$ will map $\Z = \{1, 2\}$ into $\R$. Furthermore, the spectral embedding $\m{X}_\m{B}$ has identical rows meaning that $\m{X}_i = \phi(\m{Z}_i)$ is the same for all $i$. The implication of Theorem~\ref{Thm:2Inf} is that point clouds representing each community will concentrate about the same position. While this does not bode well for separating the communities, Corollary~\ref{Cor:WSBM_CLT} demonstrates that the point cloud can nevertheless be distinguished as the combination of two communities, rather than one.

We generate a weighted stochastic block model with $n = 1000$ nodes, $\sigma^2_{k \ell} = 2$ if $k = 1, \ell = 1$, and $\sigma^2_{k \ell} = 1$ otherwise, with community assignment distribution $F$ such that $\pi_1 = 1/2$ and $\pi_2 = 1/2$. Figure~\ref{Fig:ASE_Gauss_2SBM} shows the spectral embedding of the weighted stochastic block model in $d = 1$ dimensions, the rank of the block mean matrix.
\begin{figure}[ht]
	\centering
	\includegraphics[width=0.33\textwidth]{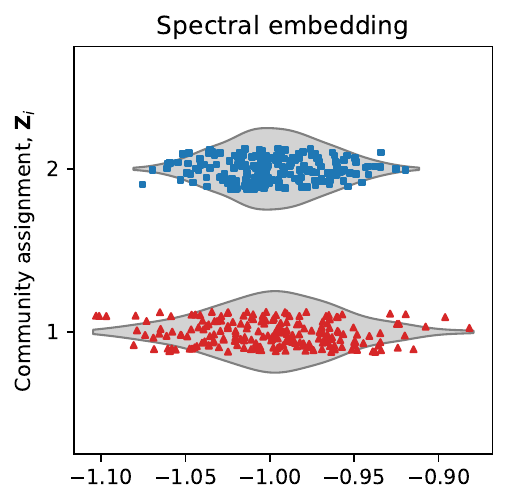}
	\caption{The embedding of a two-community weighted stochastic block model into $d = 1$ dimensions. The violin plot show the empirical distribution of the points separated by their true community assignment $\m{Z}_i$. A sample of 200 jittered points are shown colored according to their true community assignment $\m{Z}_i$.}
	\label{Fig:ASE_Gauss_2SBM}
\end{figure}
The empirical distributions for the two communities are approximately centred around -1.0 but the variances are different; the embedding for nodes with $\m{Z}_i = 1$ has sample variance 0.0017 while the embedding for nodes with $\m{Z}_i = 2$ has sample variance 0.0010. These values are very close to the covariance terms in the central limit theorem given by Corollary~\ref{Cor:WSBM_CLT}, $\m{\Sigma}_1 / n = 0.0015$ and $\m{\Sigma}_2 /n = 0.001$.

\section{Edge weight representation}
\label{Sec:CI}
In a weighted graph, the specific numeric weight assigned to each edge can often seem arbitrary. Indeed, if $\m{A}$ follows a weighted stochastic block model then so does an entry-wise transformation of $\m{A}$, yet the resulting embedding is changed, raising the question of which to use. We now describe a theoretical framework to compare representations, which is used in Section~\ref{Sec:Applications} to provide guidance about different transformations to consider, depending on the application. While we do not advise against a reasonable level of data familiarisation in practice, we should make clear that our theory does not formally allow a ``data-dependent'' choice.



\subsection{Chernoff information for weighted stochastic block models}

Under a stochastic block model, a natural way to evaluate a proposed representation of the edge weights is to measure how well communities are separated in the resulting embedding. We follow the rationale of \citet{tang2018limit} where, in a simple graph setting, Chernoff information was proposed as a measure of community separation for comparing spectral embeddings corresponding to the graph's adjacency and normalized Laplacian matrices. \citet{tang2018limit} make the point that a simpler criterion such as cluster variance is not satisfactory, since it effectively measures cluster separation assuming $K$-means clustering, inappropriate under a Gaussian mixture model with non-spherical components, as occurs here by Corollary~\ref{Cor:WSBM_CLT}. 




\begin{defn}[Chernoff information]
\label{Def:CI}
Let $F_1$ and $F_2$ be continuous multivariate distributions on $\R^d$ with density functions $f_1$ and $f_2$ respectively. The \emph{Chernoff information} \citep{chernoff1952measure} is defined by
\begin{equation*}
	\mathcal C(F_1,F_2) = \sup_{t \in (0,1)}  \left\{ - \log \int_{\R^d} f_1^t(\m{x}) f_2^{1-t}(\m{x}) \ud \m{x} \right\},
\end{equation*}
When given three or more distributions, one reports the Chernoff information of the critical pair, $\min_{k \neq \ell} \mathcal C(F_k,F_\ell)$.
\end{defn}
The Chernoff information is the exponential rate at which the Bayes error of classifying a sample to be from $F_1$ or $F_2$ decreases asymptotically. Chernoff information is unaffected by a full rank linear transformation of the sample, in particular, by distortions by $\m{Q} \in \Opq$ which can occur under spectral embedding.

For Gaussian distributions with $F_1=\mathcal{N}(\boldsymbol{\mu}_1,\m{\Gamma}_1)$ and $F_2 = \mathcal{N}(\boldsymbol{\mu}_2,\m{\Gamma}_2)$, the Chernoff information is given by \citep{P05},
\begin{equation*}
	\mathcal C(F_1,F_2) = \sup_{t \in (0,1)} \left\{ \frac{t(1-t)}{2} (\boldsymbol{\mu}_1 - \boldsymbol{\mu}_2)^\top \m{\Gamma}_{12}^{-1}(t) (\boldsymbol{\mu}_1 - \boldsymbol{\mu}_2) + \frac{1}{2} \log \frac{|\m{\Gamma}_{12}(t)|}{|\m{\Gamma}_1|^{1-t} |\m{\Gamma}_2|^t} \right\},
	\label{Eq:Chern_Norm}
\end{equation*}
where $\m{\Gamma}_{12}(t) = (1-t) \m{\Gamma}_1 + t \m{\Gamma}_2$.

Applying this to the limiting mixture model that occurs under the stochastic block model, we find
\begin{equation*}
\mathcal C = \min_{k \neq \ell}  \sup_{t \in (0,1)} \left[ \frac{n t(1-t)}{2} \left\{ (\m{X}_{\m{B}})_k - (\m{X}_{\m{B}})_\ell \right\}^\top \m{\Sigma}_{k \ell}^{-1} (t) \left\{ (\m{X}_{\m{B}})_k - (\m{X}_{\m{B}})_\ell \right\} + \frac{1}{2} \log \frac{|\m{\Sigma}_{k \ell}(t)|}{|\m{\Sigma}_k|^{1-t} |\m{\Sigma}_{\ell}|^t} \right],
\end{equation*}
where $\m{\Sigma}_{k \ell}(t) = (1-t) \m{\Sigma}_k + t \m{\Sigma}_{\ell}$ is the convex combination of the covariance matrices $\m{\Sigma}_k$ and $\m{\Sigma}_{\ell}$ defined in Corollary~\ref{Cor:WSBM_CLT}. As the number of nodes in a weighted graph increases, the Chernoff information is dominated by the first term, growing linearly with $n$, and the task of distinguishing the block model communities becomes easier. In line with previous studies, we take the limit of the above scaled by $n^{-1}$ to obtain a measure of community separation accounting for network growth \citep{mu2020identifying}.

\begin{defn}[Size-adjusted Chernoff information]
\label{Def:SACI}
Given a weighted stochastic block model that satisfies the conditions necessary for Corollary~\ref{Cor:WSBM_CLT}, define \emph{size-adjusted Chernoff information},
\begin{equation*}
	C = \min_{k \neq \ell}  \sup_{t \in (0,1)} \left[ \frac{t(1-t)}{2} \left\{ (\m{X}_{\m{B}})_k - (\m{X}_{\m{B}})_\ell \right\}^\top \m{\Sigma}_{k \ell}^{-1} (t) \left\{ (\m{X}_{\m{B}})_k - (\m{X}_{\m{B}})_\ell \right\} \right].
\end{equation*}
\end{defn}

Returning to the Gaussian stochastic block model with equal means example in Section~\ref{sec:gaussian_sbm}, the size-adjusted Chernoff information is equal to zero. This reflects the fact that the task of separating the communities does not get arbitrarily easy as $n \rightarrow \infty$. However, when the block mean matrix $\m{B}$ of a weighted stochastic block model is full rank, the size-adjusted Chernoff information calculation is greatly simplified. 

\begin{lem}
\label{Lem:SACI}
For a $K$-community stochastic block model with full rank mean block matrix, $\rank(\m{B}) = K$, the size-adjusted Chernoff information is
\begin{equation*}
    C = \min_{k \neq \ell}  \sup_{t \in (0,1)} \left[ \frac{t(1-t)}{2} \left\{ (\m{e}_k - \m{e}_\ell)^\top \m{B} \m{\Pi} \m{S}_{k \ell}(t)^{-1} \m{B} (\m{e}_k - \m{e}_\ell) \right\} \right],
\end{equation*}
where $\m{e}_k \in \R^K$ is the standard basis vector with one in position $k$, and zero elsewhere, and $\m{S}_{k \ell}(t) = (1-t) \diag(\m{C}_k) + t \diag(\m{C}_\ell) \in \R^{K \times K}$ is a convex combination of the diagonal matrices consisting of rows of the block variance matrix $\m{C}$.
\end{lem}

Rather than having to compute the spectral embedding $\m{X}_\m{B}$ and the second moment matrix $\m{\Delta}$, the size-adjusted Chernoff information can be expressed directly in terms of the block mean and variance matrices $\m{B}$ and $\m{C}$. This makes the optimization steps in the computation much more manageable. If the block mean matrix $\m{B}$ is not full rank, the right hand side is an upper bound for the size-adjusted Chernoff information.

Chernoff information increases if \emph{any} term in $\m{C}$ decreases; and decreases if \emph{all} terms in $\m{B}$ decrease by the same factor. Heuristically, if we can achieve a large reduction in the variances without reducing the means too much, we should expect better separation of communities. In applications where $\m{A}$ has extreme values, such as large counts (see Section~\ref{Sec:Airport}), we therefore recommend applying some form of tempering (e.g. $\m{A}_{ij} \rightarrow \m{A}^{1/2}_{ij}$ or $\m{A}_{ij} \rightarrow \mathbb{I}(\m{A}_{ij} > 0) \log \m{A}_{ij}$).

There is a conjecture about the limiting behavior of $\m{X}_\m{A}$ under the stochastic block model when the embedding dimension is chosen to be larger than $d = \rank(\m{B})$ \citep{yang2020simultaneous}. In the extraneous dimensions, the points of $\m{X}_\m{A}$ corresponding to a community are independent and identically distributed isotropic Gaussian with some community-specific variance. The implication of this conjecture, if true for weighted stochastic block models, would be that there is no change in size-adjusted Chernoff information when embedding into a larger dimension. On the other hand, the size-adjusted Chernoff information typically decreases and cannot increase if the embedding dimension is chosen smaller than $d$. These observations would support existing recommendations `to err on the large side' when choosing an embedding dimension \citep{athreya2017statistical}.




\subsection{Invariance under affine transformation}
\label{Affine}
Among the weight transformations that could be considered, one might be particularly concerned about the effect of an implicit choice of origin and scale --- for example, whether to measure temperatures in Fahrenheit or Celsius in a heat network. The following result shows that community separation, as measured through size-adjusted Chernoff information, is not affected by this choice. 


\begin{lem}[Chernoff information invariance under affine transformation]
\label{Thm:Affine_SBM}
Let $\m{A}$ be a weighted stochastic block model with full rank block mean matrix $\m{B}$ and size-adjusted Chernoff information $C$. For the affine entry-wise transformation $\m{A}'_{ij} = a \m{A}_{ij} + b$, if the block mean matrix $\m{B}'$ is full rank, then $\m{A}'$ has size-adjusted Chernoff information $C$.
\end{lem}

One of the main practical implications of this result is that we can often choose a sparse matrix representation, for example, the $1-$p-values matrix in Section~\ref{Sec:p-values}.

\section{Conclusion}
\label{Sec:Conc}
This article provides a statistical interpretation of the spectral embedding of a weighted graph. Understanding the asymptotic distribution of this embedding allows an informed choice between different representations of edge weights. These determinations are made under a weighted stochastic block model, using size-adjusted Chernoff information as a way to measure the quality of community separation. For example, embedding the matrix of p-values gives inferior results to using threshold or log p-values, under a broad parameter regime. 





\bibliographystyle{apalike}
\bibliography{WGRDPG_arXiv}
\newpage
\setcounter{section}{0}
\setcounter{page}{1}
\begin{center} \Large \bf Appendix \end{center}

\section{Proofs of consistency and the central limit theorem}
We begin by establishing some standard results regarding the asymptotic behavior of the singular values of $\m{P}$, $\m{A}$ and $\m{A}-\m{P}$. Recall that for $\m{Z} \sim F$, the minimal dimensionality condition from Assumption~\ref{Ass:Reg} states that the random vector $\m{X} = \phi(\m{Z})$ has second moment matrix $\m{\Delta} = \E(\m{X} \m{X}^\top)$ with full rank $d$. Therefore, a combination of a Hoeffding-style argument and a corollary of Weyl's inequalities (\citet{HJ12}, Corollary 7.3.5) shows that the $d$ non-zero singular values $\sigma_i(\m{P})$ satisfy $\sigma_i(\m{P}) = \Omega(n)$ almost surely.  By showing that the spectral norm of $\m{A}-\m{P}$ has smaller asymptotic growth, we can once again invoke Weyl's argument to show that the top $d$ singular values $\sigma_i(\m{A})$ also satisfy $\sigma_i(\m{A}) = \Omega(n)$.


\begin{prop}
\label{spectral_bound}
$\|\m{A}-\m{P}\| = \mathrm{O}\left(n^{1/2}\log^{\alpha + {1/2}}(n)\right)$ almost surely.
\end{prop}

\begin{proof}
We will make use of a matrix analogue of the Bernstein inequality (\cite{Tropp15}, Theorem 1.6.2):

\begin{thm}[Matrix Bernstein]
\label{bernstein}

Let $\m{M}_1, \ldots, \m{M}_n \in \R^{n \times n}$ be symmetric independent random matrices satisfying $\mathbb{E}(\m{M}_k) = 0$ and $\|\m{M}_k\| \leq L$ for each $1 \leq k \leq n$, for some fixed value $L$. 

Let $\m{M} = \displaystyle\sum_{k=1}^n \m{M}_k$ and let $v(\m{M}) = \|\mathbb{E}(\m{M}\m{M}^\top)\|$ denote the matrix variance statistic of $\m{M}$.  Then for all $t \geq 0$:

\begin{equation*}
	\Prob(\|\m{M}\| \geq t) \leq 2n\exp\left(\frac{-t^2/2}{v(\m{M}) + Lt/3}\right).\end{equation*}
\end{thm}

We apply this theorem as follows: for each $1 \leq i \leq j \leq n$, let $\m{M}_{ij}$ be the $n \times n$ matrix with $(i,j)^\mathrm{th}$ and $(j,i)^\mathrm{th}$ entries equal to $\m{A}_{ij} - \m{P}_{ij}$, and all other entries equal to $0$.  Then
\begin{equation*}
	\|\m{M}_{ij}\| = \left|\m{A}_{ij}-\m{P}_{ij}\right| < 2\beta\log^\alpha(n)
\end{equation*}
almost surely, and $\mathbb{E}(\m{M}_{ij}) = 0$, and so the matrix $\m{M} = \sum \m{M}_{ij} = \m{A}-\m{P}$ satisfies the criteria for Bernstein's theorem.

To bound the matrix variance statistic $v(\m{M})$, observe that
\begin{equation*}
	(\m{M}\m{M}^\top)_{ij} = \sum_{k = 1}^{n} (\m{A}_{ik} - \m{P}_{ik})(\m{A}_{jk} - \m{P}_{jk}),
\end{equation*}
and thus
\begin{equation*}
	\mathbb{E}\{(\m{M}\m{M}^\top)_{ij}\} = \left\{\begin{array}{cc}\displaystyle\sum_{k = 1}^{n} \mathrm{Var}(\m{A}_{ij})&i = j\\[4pt]0&i \neq j\end{array}\right.
\end{equation*} 

By Popoviciu's inequality, the variances $\mathrm{Var}(\m{A}_{ij})$ are bounded in absolute value by $\beta^2 \log^{2\alpha}(n)$, and so, since the matrix $\mathbb{E}(\m{M}\m{M}^\top)$ is diagonal, we see that
\begin{equation*}
	v(\m{M}) \leq \beta^2 n\log^{2\alpha}(n)
\end{equation*}
almost surely, and after substituting into Theorem~\ref{bernstein} we find that for any $t \geq 0$,
\begin{equation*}
	\Prob(\|\m{A} -\m{P}\| \geq t) \leq 2n \exp\left(\frac{-3t^2}{6\beta^2n	\log^{2\alpha}(n) + 4\beta\log^\alpha(n)t}\right)
\end{equation*}
almost surely.

The numerator of the exponential term dominates for $n$ sufficiently large if $t = cn^{1/2}\log^{\alpha+{1/2}}(n)$, and therefore $\|\m{A}-\m{P}\| = \mathrm{O}\left(n^{1/2}\log^{\alpha+{1/2}}(n)\right)$ almost surely, as required.
\end{proof}

The following result follows from an identical argument as that used in the proof of Lemma 17 in \cite{lyzinski}:

\begin{prop}
\label{orth_spectral_bound}
$\|\m{U}_\m{P}^\top(\m{A}-\m{P})\m{U}_\m{P}\|_F = \mathrm{O}\left(\log^{\alpha+{1/2}}(n)\right)$ almost surely.
\end{prop}

\begin{prop}
\label{alignment_bounds}
The following bounds hold almost surely:
\begin{enumerate}[label=\emph{\textbf{\roman*.}}, labelindent=0pt, labelwidth=!]
\item $\|\m{U}_\m{A}\m{U}_\m{A}^\top - \m{U}_\m{P}\m{U}_\m{P}^\top\| = \mathrm{O}\left(n^{-{1/2}}\log^{\alpha+{1/2}}(n)\right)$;
\item $\|\m{U}_\m{A}-\m{U}_\m{P}\m{U}_\m{P}^\top\m{U}_\m{A}\|_F = \mathrm{O}\left(n^{-{1/2}}\log^{\alpha+{1/2}}(n)\right)$;
\item $\|\m{U}_\m{P}^\top\m{U}_\m{A}\m{\Lambda}_\m{A}-\m{\Lambda}_\m{P}\m{U}_\m{P}^\top\m{U}_\m{A}\|_F = \mathrm{O}\left(\log^{2\alpha+1}(n)\right)$; 
\item $\|\m{U}_\m{P}^\top\m{U}_\m{A}\Ipq - \Ipq\m{U}_\m{P}^\top\m{U}_\m{A}\|_F = \mathrm{O}\left(n^{-1}\log^{2\alpha+1}(n)\right)$
\end{enumerate}
\end{prop}

\begin{proof}
\
\begin{enumerate}[label=\emph{\textbf{\roman*.}}, labelindent=0pt, labelwidth=!]
\item Let $\sigma_1, \ldots, \sigma_d$ denote the singular values of $\m{U}_\m{P}^\top\m{U}_\m{A}$, and let $\theta_i = \cos^{-1}(\sigma_i)$ be the principal angles.  It is a standard result that the non-zero eigenvalues of the matrix $\m{U}_\m{A}\m{U}_\m{A}^\top - \m{U}_\m{P}\m{U}_\m{P}^\top$ are precisely the $\sin(\theta_i)$ (each occurring twice) and so, by a variant of Davis-Kahan (\cite{davis_kahan}, Theorem 4) we have
\begin{equation*}
\|\m{U}_\m{A}\m{U}_\m{A}^\top - \m{U}_\m{P}\m{U}_\m{P}^\top\| = \m{A}x_{i \in \{1,\ldots,d\}}|\sin(\theta_i)| \leq \frac{2\sqrt{d}\left(2\sigma_1(\m{P})+\|\m{A}-\m{P}\|\right)\|\m{A}-\m{P}\|}{\sigma_{d}(\m{P})^2}
\end{equation*}
for $n$ sufficiently large.

The spectral norm bound from Proposition~\ref{spectral_bound} then shows that
\begin{align*}
\|\m{U}_\m{A}\m{U}_\m{A}^\top - \m{U}_\m{P}\m{U}_\m{P}^\top\| &= \mathrm{O}\left(\frac{\left\{\sigma_1(\m{P})+n^{1/2}\log^{\alpha+{1/2}}(n)\right\} n^{1/2}\log^{\alpha+{1/2}}(n)}{\sigma_{d}(\m{P})^2}\right)
\\&= \mathrm{O}\left(n^{-{1/2}}\log^{\alpha+{1/2}}(n)\right)
\end{align*}
since $\sigma_i(\m{P}) = \Omega(n)$ almost surely.

\item Using the bound from part \emph{\textbf{i.}}, we find that
\begin{align*}
\|\m{U}_\m{A}-\m{U}_\m{P}\m{U}_\m{P}^\top\m{U}_\m{A}\|_F &= \|(\m{U}_\m{A}\m{U}_\m{A}^\top-\m{U}_\m{P}\m{U}_\m{P}^\top)\m{U}_\m{A}\|_F \leq \|\m{U}_\m{A}\m{U}_\m{A}^\top-\m{U}_\m{P}\m{U}_\m{P}^\top\|\|\m{U}_\m{A}\|_F
\\&= \mathrm{O}\left(n^{-{1/2}}\log^{\alpha+{1/2}}(n)\right).
\end{align*}

\item Observe that
\begin{align*}
\m{U}_\m{P}^\top\m{U}_\m{A}\m{\Lambda}_\m{A}-\m{\Lambda}_\m{P}\m{U}_\m{P}^\top\m{U}_\m{A}\Ipq &= \m{U}_\m{P}^\top(\m{A}-\m{P})\m{U}_\m{A}
\\&= \m{U}_\m{P}^\top(\m{A}-\m{P})(\m{U}_\m{A} - \m{U}_\m{P}\m{U}_\m{P}^\top\m{U}_\m{A}) + \m{U}_\m{P}^\top(\m{A}-\m{P})\m{U}_\m{P}\m{U}_\m{P}^\top\m{U}_\m{A},
\end{align*} 
and so
\begin{eqnarray*}
\lefteqn{ \|\m{U}_\m{P}^\top\m{U}_\m{A}\m{\Lambda}_\m{A}-\m{\Lambda}_\m{P}\m{U}_\m{P}^\top\m{U}_\m{A}\|_F } \\
&\leq& \|\m{U}_\m{P}^\top\|\|\m{A}-\m{P}\|\|\m{U}_\m{A} - \m{U}_\m{P}\m{U}_\m{P}^\top\m{U}_\m{A}\|_F + \|\m{U}_\m{P}^\top(\m{A}-\m{P})\m{U}_\m{P}\|_F\|\m{U}_\m{P}^\top\m{U}_\m{A}\|_F
\\&=& \mathrm{O}\left(n^{1/2}\log^{\alpha+\frac{1}{2}}(n) \cdot n^{-{1/2}}\log^{\alpha+{1/2}}(n)\right) + \mathrm{O}\left(\log^{\alpha+{1/2}}(n)\right)
\\&=& \mathrm{O}\left(\log^{2\alpha+1}(n)\right),
\end{eqnarray*}
where we have used Propositions~\ref{spectral_bound}, \ref{orth_spectral_bound} and the result from part \emph{\textbf{ii.}}.

\item Note that
\begin{align*}
\m{U}_\m{P}^\top\m{U}_\m{A}\Ipq - \Ipq\m{U}_\m{P}^\top\m{U}_\m{A} =& \left\{(\m{U}_\m{P}^\top\m{U}_\m{A}\m{\Lambda}_\m{A}-\m{\Lambda}_\m{P}\m{U}_\m{P}^\top\m{U}_\m{A}) + (\m{\Sigma}_\m{P}\m{U}_\m{P}^\top\m{U}_\m{A} - \Ipq\m{U}_\m{P}^\top\m{U}_\m{A}\m{\Lambda}_\m{A})\right\}\m{\Sigma}_\m{A}^{-1} \\&- \m{\Sigma}_\m{P}(\m{U}_\m{P}^\top\m{U}_\m{A}\Ipq - \Ipq\m{U}_\m{P}^\top\m{U}_\m{A})\m{\Sigma}_\m{A}^{-1}.
\end{align*}
where $\m{\Sigma}_\m{A} = \m{\Lambda}_\m{A} \Ipq$ and $\m{\Sigma}_\m{P} = \m{\Lambda}_\m{P} \Ipq$.

For any $i, j \in \{1,\ldots,d\}$, by rearranging and bounding the absolute value of the right-hand terms by the Frobenius norm, we find
\begin{eqnarray*}
\lefteqn{ \left|(\m{U}_\m{P}^\top\m{U}_\m{A}\Ipq - \Ipq\m{U}_\m{P}^\top\m{U}_\m{A})_{ij}\right|\left(1 + \tfrac{\sigma_i(\m{P})}{\sigma_j(\m{A})}\right)}
\\&\leq& \left(\|\m{U}_\m{P}^\top\m{U}_\m{A}\m{\Lambda}_\m{A}-\m{\Lambda}_\m{P}\m{U}_\m{P}^\top\m{U}_\m{A}\|_F + \|\m{\Sigma}_\m{P}\m{U}_\m{P}^\top\m{U}_\m{A} - \Ipq\m{U}_\m{P}^\top\m{U}_\m{A}\m{\Lambda}_\m{A}\|_F\right)\|\m{\Sigma}_\m{A}^{-1}\|_F
\\&=& \left(\|\m{U}_\m{P}^\top\m{U}_\m{A}\m{\Lambda}_\m{A}-\m{\Lambda}_\m{P}\m{U}_\m{P}^\top\m{U}_\m{A}\|_F + \|\m{\Lambda}_\m{P}\m{U}_\m{P}^\top\m{U}_\m{A} - \m{U}_\m{P}^\top\m{U}_\m{A}\m{\Lambda}_\m{A}\|_F\right)\|\m{\Sigma}_\m{A}^{-1}\|_F
\\&=& \mathrm{O}\left(n^{-1}\log^{2\alpha+1}(n)\right),
\end{eqnarray*}
where we have used part \emph{\textbf{iii.}}  The result follows from the fact that $\left(1 + \tfrac{\sigma_i(\m{P})}{\sigma_j(\m{A})}\right) \geq 1$.
\end{enumerate}
\end{proof}

\begin{prop}
\label{procrustes_bound}
Let $\m{U}_\m{P}^\top\m{U}_\m{A} + \Ipq\m{U}_\m{P}^\top\m{U}_\m{A}\Ipq$ admit the singular value decomposition
\begin{align*}
\m{U}_\m{P}^\top\m{U}_\m{A} + \Ipq\m{U}_\m{P}^\top\m{U}_\m{A}\Ipq = \m{W}_1\m{\Sigma}\m{W}_2^\top,
\end{align*}
and let $\m{W} = \m{W}_1\m{W}_2^\top$.  Then $\m{W} \in \mathbb{O}(d) \cap \mathbb{O}(p,q)$ and
\begin{align*}
\|\m{U}_\m{P}^\top\m{U}_\m{A} - \m{W}\|_F, ~\|\Ipq\m{U}_\m{P}^\top\m{U}_\m{A}\Ipq - \m{W}\|_F = \mathrm{O}\left(n^{-1}\log^{2\alpha+1}(n)\right)
\end{align*}
almost surely.
\end{prop}

\begin{proof}
A standard argument shows that a solution to the modified \emph{one mode} orthogonal Procrustes problem
\begin{align*}
\widehat{\m{W}} = \argmin_{\m{Q} \in \mathbb{O}(d)} \|\m{P}_1^\top\m{A}_1 - \m{Q}\|_F^2 + \|\m{P}_2^\top\m{A}_2 - \m{Q}\|_F^2
\end{align*}
for matrices $\m{A}_i, \m{P}_i \in \R^{n \times d}$ is given by $\widehat{\m{W}} = \widehat{\m{W}}_1 \widehat{\m{W}}_2^\top$, where we have the singular value decomposition
\begin{align*}
\tfrac{1}{2}(\m{P}_1^\top\m{A}_1 + \m{P}_2\m{A}_2^\top) = \widehat{\m{W}}_1 \m{\Sigma} \widehat{\m{W}}_2^\top.
\end{align*}

Setting $\m{W}$ as in the statement of the proposition, we therefore observe that $\m{W}$ satisfies
\begin{align*}
\m{W} = \argmin_{\m{Q} \in \mathbb{O}(d)}\|\m{U}_\m{P}^\top\m{U}_\m{A} - \m{Q}\|_F^2 + \|\Ipq\m{U}_\m{P}^\top\m{U}_\m{A}\Ipq - \m{Q}\|_F^2.
\end{align*}

Let $\m{U}_\m{P}^\top\m{U}_\m{A} = \m{W}_{\m{U},1}\m{\Sigma}_\m{U}\m{W}_{\m{U},2}^\top$ be the singular value decomposition of $\m{U}_\m{P}^\top\m{U}_\m{A}$, and define $\m{W}_\m{U} \in \mathbb{O}(d)$ by $\m{W}_\m{U} = \m{W}_{\m{U},1}\m{W}_{\m{U},2}^\top$.  Then
\begin{align*}
\|\m{U}_\m{P}^\top \m{U}_\m{A} - \m{W}_\m{U}\|_F &= \|\m{\Sigma} - \m{I}\|_F = \left(\sum_{i=1}^d (1-\sigma_i)^2\right)^{1/2} \leq \sum_{i=1}^d (1-\sigma_i) \leq \sum_{i=1}^d(1-\sigma_i^2) 
\\&= \sum_{i=1}^d \sin^2(\theta_i) \leq d\|\m{U}_\m{A}\m{U}_\m{A}^\top - \m{U}_\m{P}\m{U}_\m{P}^\top\|^2 \\&= \mathrm{O}\left(n^{-1}\log^{2\alpha+1}(n)\right).
\end{align*}

Also,
\begin{align*}
\|\Ipq\m{U}_\m{P}^\top \m{U}_\m{A}\Ipq - \m{W}_\m{U}\|_F &\leq \|\Ipq\m{U}_\m{P}^\top \m{U}_\m{A}\Ipq - \m{U}_\m{P}^\top\m{U}_\m{A}\|_F + \|\m{U}_\m{P}^\top\m{U}_\m{A} - \m{W}_\m{U}\|_F \\& \leq \|\m{U}_\m{P}^\top \m{U}_\m{A}\Ipq - \Ipq\m{U}_\m{P}^\top\m{U}_\m{A}\|_F + \|\m{U}_\m{P}^\top\m{U}_\m{A} - \m{W}_\m{U}\|_F \\&= \mathrm{O}\left(n^{-1}\log^{2\alpha+1}(n)\right)
\end{align*}
by Proposition~\ref{alignment_bounds}.

Combining these shows that
\begin{align*}
\|\m{U}_\m{P}^\top\m{U}_\m{A} - \m{W}\|_F^2 + \|\Ipq\m{U}_\m{P}^\top\m{U}_\m{A}\Ipq - \m{W}\|_F^2 &\leq \|\m{U}_\m{P}^\top\m{U}_\m{A} - \m{W}_\m{U}\|_F^2 + \|\Ipq\m{U}_\m{P}^\top\m{U}_\m{A}\Ipq - \m{W}_\m{U}\|_F^2
\\&= \mathrm{O}\left(n^{-2}\log^{4\alpha+2}(n)\right),
\end{align*}
which gives the desired bound. 

Finally, we observe that the matrix $\m{U}_\m{P}^\top\m{U}_\m{A} + \Ipq\m{U}_\m{P}^\top\m{U}_\m{A}\Ipq \in \R^{p \times p} \oplus \R^{q \times q}$, and thus the matrices $\m{W}_1, \m{W}_2 \in \mathbb{O}(p) \oplus \mathbb{O}(q)$, so in particular $\m{W} \in \mathbb{O}(d) \cap \mathbb{O}(p,q)$.
\end{proof}

\begin{prop}
\label{W_bounds}
The following bounds hold almost surely:
\begin{enumerate}[label=\emph{\textbf{\roman*.}}, labelindent=0pt, labelwidth=!]
\item $\|\m{W}\m{\Sigma}_\m{A} - \m{\Sigma}_\m{P}\m{W}\|_F = \mathrm{O}\left(\log^{2\alpha+1}(n)\right)$;
\item $\|\m{W}\m{\Sigma}_\m{A}^{1/2} - \m{\Sigma}_\m{P}^{1/2}\m{W}\|_F = \mathrm{O}\left(n^{-{1/2}}\log^{2\alpha+1}(n)\right)$;
\item $\|\m{W}\m{\Sigma}_\m{A}^{-{1/2}} - \m{\Sigma}_\m{P}^{-{1/2}}\m{W}\|_F = \mathrm{O}\left(n^{-{3/2}}\log^{2\alpha+1}(n)\right)$.
\end{enumerate}
\end{prop}

\begin{proof}
\begin{enumerate}[wide, label=\emph{\textbf{\roman*.}}, labelindent=0pt, labelwidth=!]
\item Observe that
\begin{align*}
\m{W}\m{\Sigma}_\m{A} - \m{\Sigma}_\m{P}\m{W} &= (\m{W}-\m{U}_\m{P}^\top\m{U}_\m{A})\m{\Sigma}_\m{A} + \m{U}_\m{P}^\top\m{U}_\m{A}\m{\Sigma}_\m{A} - \m{\Sigma}_\m{P}\m{W}
\\&= (\m{W}-\m{U}_\m{P}^\top\m{U}_\m{A})\m{\Sigma}_\m{A} + (\m{U}_\m{P}^\top\m{U}_\m{A}\m{\Sigma}_\m{A} - \m{\Sigma}_\m{P}\Ipq\m{U}_\m{P}^\top\m{U}_\m{A}\Ipq) + \m{\Sigma}_\m{P}(\Ipq\m{U}_\m{P}^\top\m{U}_\m{A}\Ipq - \m{W}).
\end{align*}

Proposition~\ref{procrustes_bound} shows that the terms $\|(\m{W}-\m{U}_\m{P}^\top\m{U}_\m{A})\m{\Sigma}_\m{A}\|_F$ and $\|\m{\Sigma}_\m{P}(\Ipq\m{U}_\m{P}^\top\m{U}_\m{A}\Ipq-\m{W})\|_F$ are both $\mathrm{O}\left(\log^{2\alpha+1}(n)\right)$, while $\|\m{U}_\m{P}^\top\m{U}_\m{A}\m{\Sigma}_\m{A} - \m{\Sigma}_\m{P}\Ipq\m{U}_\m{P}^\top\m{U}_\m{A}\Ipq\|_F$ is $\mathrm{O}\left(\log^{2\alpha+1}(n)\right)$, and so $\|\m{W}\m{\Sigma}_\m{A} - \m{\Sigma}_\m{P}\m{W}\|_F = \mathrm{O}\left(\log^{2\alpha+1}(n)\right)$.

\item We will bound the absolute value of the terms $\left(\m{W}\m{\Sigma}_\m{A}^{1/2}-\m{\Sigma}_\m{P}^{1/2} \m{W}\right)_{ij}$. Note that 
\begingroup
\addtolength{\jot}{1em}
\begin{align*}
\left|\left(\m{W}\m{\Sigma}_\m{A}^{1/2}-\m{\Sigma}_\m{P}^{1/2} \m{W}\right)_{ij}\right| &= \left|\m{W}_{ij}\left(\sigma_j(\m{A})^{1/2} - \sigma_i(\m{P})^{1/2}\right)\right| = \left|\frac{\m{W}_{ij}(\sigma_j(\m{A}) - \sigma_i(\m{P}))}{\sigma_j(\m{A})^{1/2} + \sigma_i(\m{P})^{1/2}}\right| \\&=  \frac{\left|\left(\m{W}\m{\Sigma}_\m{A} - \m{\Sigma}_\m{P}\m{W}\right)_{ij}\right|}{\sigma_j(\m{A})^{1/2} + \sigma_i(\m{P})^{1/2}} \leq \frac{\|\m{W}\m{\Sigma}_\m{A} - \m{\Sigma}_\m{P}\m{W}\|_F}{\sigma_d(\m{P})^{1/2}},
\end{align*}
\endgroup
and consequently we find that $\|\m{W}\m{\Sigma}_\m{A}^{1/2} - \m{\Sigma}_\m{P}^{1/2}\m{W}\|_F = \mathrm{O}\left(n^{-{1/2}}\log^{2\alpha+1}(n)\right)$ by summing over all $i, j \in \{1,\ldots,d\}$ and applying part \emph{\textbf{i.}}

\item We will bound the absolute value of the terms $\left(\m{W}\m{\Sigma}_\m{A}^{-{1/2}}-\m{\Sigma}_\m{P}^{-{1/2}} \m{W}\right)_{ij}$. Note that 
\begingroup
\addtolength{\jot}{1em}
\begin{align*}
\left|\left(\m{W}\m{\Sigma}_\m{A}^{-{1/2}}-\m{\Sigma}_\m{P}^{-{1/2}} \m{W}\right)_{ij}\right| &= \left|\frac{\m{W}_{ij}(\sigma_i(\m{P})^{1/2} - \sigma_j(\m{A})^{1/2})}{\sigma_i(\m{P})^{1/2}\sigma_j(\m{P})^{1/2}}\right| \\&= \frac{\left|\left(\m{W}\m{\Sigma}_\m{A}^{1/2}-\m{\Sigma}_\m{P}^{1/2} \m{W}\right)_{ij}\right|}{\sigma_i(\m{P})^{1/2}\sigma_j(\m{A})^{1/2}} = \mathrm{O}\left(n^{-{3/2}}\log^{2\alpha+1}(n)\right)
\end{align*}
\endgroup
by part \emph{\textbf{ii.}}  The result follows by summing over all $i, j \in \{1, \ldots, d\}$.
\end{enumerate} 
\end{proof}

\begin{prop}
\label{residual_bounds}
Let
\begin{align*}
\m{R}_1 &= \m{U}_\m{P}(\m{U}_\m{P}^\top\m{U}_\m{A}\m{\Sigma}_\m{A}^{1/2} - \m{\Sigma}_\m{P}^{1/2} \m{W})
\\\m{R}_2 &= (\m{I} - \m{U}_\m{P}\m{U}_\m{P}^\top)(\m{A}-\m{P})(\m{U}_\m{A}\Ipq - \m{U}_\m{P}\Ipq\m{W})\m{\Sigma}_\m{A}^{-{1/2}}
\\\m{R}_3 &= -\m{U}_\m{P}\m{U}_\m{P}^\top(\m{A}-\m{P})\m{U}_\m{P}\Ipq\m{W}\m{\Sigma}_\m{A}^{-{1/2}}
\\\m{R}_4 &= (\m{A}-\m{P})\m{U}_\m{P}\Ipq(\m{W}\m{\Sigma}_\m{A}^{-{1/2}} - \m{\Sigma}_\m{P}^{-{1/2}}\m{W}).
\end{align*}

Then the following bounds hold almost surely:
\begin{enumerate}[label=\emph{\textbf{\roman*.}}]
\item $\|\m{R}_1\|_{2 \to \infty} = \mathrm{O}\left(n^{-1}\log^{2\alpha+1}(n)\right)$;
\item $\|\m{R}_2\|_{2 \to \infty} = \mathrm{O}\left(n^{-\frac{3}{4}}\log^{3\alpha+{3/2}}(n)\right)$;
\item $\|\m{R}_3\|_{2 \to \infty} = \mathrm{O}\left(n^{-1}\log^{\alpha+{1/2}}(n)\right)$;
\item $\|\m{R}_4\|_{2 \to \infty} = \mathrm{O}\left(n^{-1}\log^{3\alpha+{3/2}}(n)\right)$
\end{enumerate}
In particular, we have $\|n^{1/2}\m{R}_i\|_{2 \to \infty} \to 0$ for all $i$.
\end{prop}

\begin{proof}
\begin{enumerate}[wide, label=\emph{\textbf{\roman*.}}, labelindent=0pt, labelwidth=!]
\item Recall that $\m{U}_\m{P}\m{\Sigma}_\m{P}^{1/2} = \m{X}\m{Q}_\m{X}$ for some $\m{Q}_\m{X} \in \mathbb{O}(p,q)$ of bounded spectral norm.  Using the relation $\|\m{A}\m{P}\|_{2 \to \infty} \leq \|\m{A}\|_{2 \to \infty} \|\m{P}\|$ (see, for example, \cite{two_to_infinity}, Proposition 6.5) we find that $\|\m{U}_\m{P}\|_{2 \to \infty} \leq \|\m{X}\|_{2 \to \infty}\|\m{Q}_\m{X}\|\|\m{\Sigma}_\m{P}^{-1}\|$, and thus $\|\m{U}_\m{P}\|_{2 \to \infty} = \mathrm{O}\left(n^{-{1/2}}\right)$ as the rows of $\m{X}$ are by definition bounded in Euclidean norm.

Thus
\begin{align*}
\|\m{R}_1\|_{2 \to \infty} &\leq \|\m{U}_\m{P}\|_{2 \to \infty}\|\m{U}_\m{P}^\top\m{U}_\m{A}\m{\Sigma}_\m{A}^{1/2} - \m{\Sigma}_\m{P}^{1/2} \m{W}\|
\\& \leq \|\m{U}_\m{P}\|_{2\to\infty}\left(\|(\m{U}_\m{P}^\top\m{U}_\m{A} - \m{W})\m{\Sigma}_\m{A}^{1/2}\|_F + \|\m{W}\m{\Sigma}_\m{A}^{1/2}-\m{\Sigma}_\m{P}^{1/2} \m{W}\|_F\right)
\end{align*}
The first summand is $\mathrm{O}\left(n^{-{1/2}}\log^{2\alpha+1}(n)\right)$ by Proposition~\ref{procrustes_bound}, while Proposition~\ref{W_bounds} shows that the second is $\mathrm{O}\left(n^{-{1/2}}\log^{2\alpha+1}(n)\right)$, and so
\begin{align*}
\|\m{R}_1\|_{2 \to \infty} = \mathrm{O}\left(n^{-1}\log^{2\alpha+1}(n)\right).
\end{align*}

\item We first observe that
\begin{eqnarray*}
\lefteqn{\|\m{U}_\m{P}\m{U}_\m{P}^\top(\m{A}-\m{P})(\m{U}_\m{A}\Ipq-\m{U}_\m{P}\Ipq\m{W})\m{\Sigma}_\m{A}^{-{1/2}}\|_{2 \to \infty}}
\\&\leq& \|\m{U}_\m{P}\|_{2\to \infty}\|\m{U}_\m{P}^\top\|\|\m{A}-\m{P}\|\|\m{U}_\m{A}\Ipq-\m{U}_\m{P}\Ipq\m{W}\|\|\m{\Sigma}_\m{A}^{-{1/2}}\|
\\&=& \mathrm{O}\left(n^{-{1/2}} \cdot n^{1/2}\log^{\alpha+{1/2}}(n) \cdot n^{-{1/2}}\log^{2\alpha+1}(n)\cdot n^{-{1/2}}\right)
\\&=& \mathrm{O}\left(n^{-1}\log^{3\alpha+{3/2}}(n)\right),
\end{eqnarray*}
where we have bounded $\|\m{U}_\m{A}\Ipq-\m{U}_\m{P}\Ipq\m{W}\|$ by noting that 
\begin{align*}
\|\m{U}_\m{A}\Ipq-\m{U}_\m{P}\Ipq\m{W}\| &\leq \|\m{U}_\m{A}\Ipq-\m{U}_\m{P}\m{U}_\m{P}^\top\m{U}_\m{A}\Ipq\| + \|\m{U}_\m{P}\Ipq(\Ipq\m{U}_\m{P}^\top\m{U}_\m{A}\Ipq-\m{W})\| \\&= \mathrm{O}\left(n^{-{1/2}}\log^{2\alpha+1}(n)\right),\end{align*}
by Propositions~\ref{alignment_bounds} and \ref{procrustes_bound}.

This leaves us to bound the term $\|(\m{A}-\m{P})(\m{U}_\m{A}\Ipq-\m{U}_\m{P}\Ipq\m{W})\m{\Sigma}_\m{A}^{-{1/2}}\|_{2 \to \infty}$.  Now,
\begin{eqnarray*}
(\m{A}-\m{P})(\m{U}_\m{A}\Ipq-\m{U}_\m{P}\Ipq\m{W})\m{\Sigma}_\m{A}^{-{1/2}} &=& (\m{A}-\m{P})(\m{I}-\m{U}_\m{P}\m{U}_\m{P}^\top)\m{U}_\m{A}\Ipq\m{\Sigma}_\m{A}^{-{1/2}} 
\\&&+ (\m{A}-\m{P})\m{U}_\m{P}\Ipq(\Ipq\m{U}_\m{P}^\top\m{U}_\m{A}\Ipq-\m{W})\m{\Sigma}_\m{A}^{-{1/2}},
\end{eqnarray*}
and
\begin{align*}
\|(\m{A}-\m{P})\m{U}_\m{P}\Ipq(\Ipq\m{U}_\m{P}^\top\m{U}_\m{A}\Ipq-\m{W})\m{\Sigma}_\m{A}^{-{1/2}}\|_{2 \to \infty} &\leq \|\m{A}-\m{P}\|\|\m{U}_\m{P}\|\|\Ipq\m{U}_\m{P}^\top\m{U}_\m{A}\Ipq-\m{W}\|\|\m{\Sigma}_\m{A}^{-{1/2}}\| 
\\& = \mathrm{O}\left(n^{1/2}\log^{\alpha+{1/2}}(n) \cdot n^{-1}\log^{2\alpha+1}(n) \cdot n^{-{1/2}}\right)
\\& = \mathrm{O}\left(n^{-1}\log^{3\alpha+{3/2}}(n)\right)
\end{align*}
by Propositions~\ref{spectral_bound} and \ref{procrustes_bound}.

To bound the remaining term, observe that we can rewrite
\begin{align*}
(\m{A}-\m{P})(\m{I}-\m{U}_\m{P}\m{U}_\m{P}^\top)\m{U}_\m{A}\Ipq\m{\Sigma}_\m{A}^{-{1/2}} = (\m{A}-\m{P})(\m{I}-\m{U}_\m{P}\m{U}_\m{P}^\top)\m{U}_\m{A}\m{U}\m{A}^\top\m{U}_\m{A}\Ipq\m{\Sigma}_\m{A}^{-{1/2}}
\end{align*}
and so
\begin{align*}
\|(\m{A}-\m{P})(\m{I}-\m{U}_\m{P}\m{U}_\m{P}^\top)\m{U}_\m{A}\Ipq\m{\Sigma}_\m{A}^{-{1/2}}\|_{2 \to \infty} & \leq \|\m{R}\|_{2 \to \infty}\|\m{U}_\m{A}\Ipq\m{\Sigma}_\m{A}^{-{1/2}}\|,
\end{align*}
where $\m{R} = (\m{A}-\m{P})(\m{I}-\m{U}_\m{P}\m{U}_\m{P}^\top)\m{U}_\m{A}\m{U}_\m{A}^\top$.

The latter term is $\mathrm{O}\left(n^{-{1/2}}\right)$, so it suffices to bound $\|\m{R}\|_{2 \to \infty}$.  To do this, we claim that the Frobenius norms of the rows of the matrix $\m{R}$ are exchangeable, and thus have the same expectation, which implies that $\E\left(\|\m{R}\|_F^2\right) = n\E\left(\|\m{R}_i\|^2\right)$ for any $i \in \{1,\ldots,n\}$.  Applying Markov's inequality, we therefore see that
\begin{align*}
\Prob\left(\|\m{R}_i\| > t\right) \leq \frac{\E\left(\|\m{R}_i\|^2\right)}{t^2} = \frac{\E\left(\|\m{R}\|_F^2\right)}{nt^2}.
\end{align*} 

Now,
\begin{align*}
\|\m{R}\|_F &\leq \|\m{A}-\m{P}\|\|\m{A}-\m{U}_\m{P}\m{U}_\m{P}^\top\m{A}\|_F\|\m{U}_\m{A}^\top\|_F
\\&= \mathrm{O}\left(n^{1/2}\log^{\alpha+{1/2}}(n)\cdot n^{-{1/2}}\log^{\alpha+{1/2}}(n)\right)
\\&= \mathrm{O}\left(\log^{2\alpha+1}(n)\right)
\end{align*}
by Propositions~\ref{spectral_bound} and \ref{alignment_bounds}. 

It follows that
\begin{align*}
\Prob\left(\|\m{R}_i\| > n^{-\frac{1}{4}}\log^{2\alpha+1}(n)\right) \leq cn^{-{1/2}}
\end{align*}
and thus
\begin{align*}
\|\m{R}\|_{2 \to \infty} = \mathrm{O}\left(n^{-\frac{1}{4}}\log^{2\alpha+1}(n)\right)
\end{align*}
almost surely.

We must therefore show that the Frobenius norms of the rows of $\m{R}$ are exchangeable.  Let $\m{Q} \in \mathbb{O}(n)$ be a permutation matrix, and observe that for any matrix $\m{G} \in \R^{n \times n}$, right multiplication by $\m{Q}^\top$ simply permutes the columns of $\m{G}$, and thus does not alter the Frobenius norms of its rows.  In particular, the Frobenius norms of the rows of $\m{Q}\m{G}\m{Q}^\top$ are the same as the Frobenius norms of the rows of $\m{Q}\m{G}$.  For any symmetric matrix $\m{G} \in \R^{n \times n}$, let $\mathcal{P}_d(\m{G})$ denote the projection onto the subspace spanned by the eigenvectors corresponding to the top $d$ singular values of $\m{G}$, and let $\mathcal{P}_d^\perp(\m{G})$ denote the projection onto the orthogonal complements of this subspace.

Note that
\begin{align*}
\mathcal{P}_d(\m{P}) = \m{U}_\m{P}\m{U}_\m{P}^\top \fand \mathcal{P}_d(\m{A}) = \m{U}_\m{A}\m{U}_\m{A}^\top,
\end{align*}
while for any permutation matrix $\m{Q} \in \mathbb{O}(n)$ we have
\begin{align*}
\mathcal{P}_d(\m{Q}\m{P}\m{Q}^\top) = \m{Q}\m{U}_\m{P}\m{U}_\m{P}^\top\m{Q}^\top \fand \mathcal{R}_d(\m{Q}\m{A}\m{Q}^\top) = \m{Q}\m{U}_\m{A}\m{U}_\m{A}^\top\m{Q}^\top.
\end{align*} 

For any pair of matrices $\m{G}, \m{H} \in \R^{n \times mn}$, define an operator
\begin{align*}
\widehat{\mathcal{P}}_d(\m{G},\m{H}) = (\m{G}-\m{H})\mathcal{P}_d^\perp(\m{H})\mathcal{P}_d(\m{G})
\end{align*}
and note that $\widehat{\mathcal{P}}_d(\m{A},\m{P}) = \m{R}$, while
\begin{align*}
\widehat{\mathcal{P}}_d(\m{Q}\m{A}\m{Q}^\top,\m{Q}\m{P}\m{Q}^\top) &= \m{Q}(\m{A}-\m{P})\m{Q}^\top\m{Q}(\m{I}-\m{U}_\m{P}\m{U}_\m{P}^\top)\m{Q}^\top\m{Q}\m{U}_\m{A}\m{U}_\m{A}^\top\m{Q}^\top
\\&= \m{Q}\m{R}\m{Q}^\top.
\end{align*}

By assumption, the latent positions for our graphs are independent and identically distributed, and so the entries of the pair $(\m{A},\m{P})$ have the same joint distribution as those of the pair $(\m{Q}\m{A}\m{Q}^\top,\m{Q}\m{P}\m{Q}^\top)$.  Therefore, the entries of the matrix $\mathcal{P}_\mathcal{L}(\m{A},\m{P})$ have the same joint distribution as those of the matrix $\widehat{\mathcal{P}}_d(\m{Q}\m{A}\m{Q}^\top,\m{Q}\m{P}\m{Q}^\top)$, which implies that $\m{R}$ has the same distribution as $\m{Q}\m{R}\m{Q}^\top$, and consequently the Frobenius norms of the rows of $\m{R}$ have the same distribution as those of $\m{Q}\m{R}$, which proves our claim.

Combining these results, we see that
\begin{align*}
\|\m{R}_2\|_{2 \to \infty} = \mathrm{O}\left(n^{-\frac{3}{4}}\log^{3\alpha+{3/2}}(n)\right),
\end{align*}
as required.

\item Similarly to part \emph{\textbf{i.}}, we see that
\begin{align*}
\|\m{R}_3\|_{2 \to \infty} &\leq \|\m{U}_\m{P}\|_{2 \to \infty}\|\m{U}_\m{P}^\top(\m{A}-\m{P})\m{U}_\m{P}\Ipq\m{W}\m{\Sigma}_\m{A}^{-{1/2}}\|
\\& \leq \|\m{U}_\m{P}\|_{2 \to \infty}\|\m{U}_\m{P}^\top(\m{A}-\m{P})\m{U}_\m{P}\|_F\|\m{W}\m{\Sigma}_\m{A}^{-{1/2}}\|_F
\\& = \mathrm{O}\left(n^{-{1/2}} \cdot  \log^{\alpha+{1/2}}(n) \cdot  n^{-{1/2}}\right)
\\&= \mathrm{O}\left(n^{-1}\log^{\alpha+{1/2}}(n)\right)
\end{align*}
by Proposition~\ref{orth_spectral_bound}.

\item Observe that
\begin{align*}
\|\m{R}_4\|_{2 \to \infty} &\leq \|\m{R}_4\|_F
\\& \leq \|\m{A}-\m{P}\|\|\m{U}_\m{P}\|_F\|\m{W}\m{\Sigma}_\m{A}^{-{1/2}} - \m{\Sigma}_\m{P}^{-{1/2}}\m{W}\|_F
\\& = \mathrm{O}\left(n^{-1}\log^{3\alpha+{3/2}}(n)\right)
\end{align*}
by Propositions~\ref{spectral_bound} and \ref{W_bounds}.
\end{enumerate}
\end{proof}

\subsection{Proof of Theorem~\ref{Thm:2Inf}}

\begin{proof}
Observe that 
\begin{align*}
\m{X}_\m{A} - \m{X}_\m{P}\m{W} &= \m{U}_\m{A} \m{\Sigma}_\m{A}^{1/2} - \m{U}_\m{P} \m{\Sigma}_\m{P}^{1/2} \m{W} 
\\&= \m{U}_\m{A} \m{\Sigma}_\m{A}^{1/2} - \m{U}_\m{P}\m{U}_\m{P}^\top\m{U}_\m{A}\m{\Sigma}_\m{A}^{1/2} + \m{U}_\m{P}(\m{U}_\m{P}^\top\m{U}_\m{A}\m{\Sigma}_\m{A}^{1/2} - \m{\Sigma}_\m{P}^{1/2} \m{W} )  
\\&= \m{U}_\m{A} \m{\Sigma}_\m{A}^{1/2} - \m{U}_\m{P}\m{U}_\m{P}^\top\m{U}_\m{A}\m{\Sigma}_\m{A}^{1/2} + \m{R}_{1,1}.
\end{align*} 

Noting that
\begin{align*}
\m{U}_\m{A}\m{\Sigma}_\m{A}^{1/2} = \m{A}\m{U}_\m{A}\Ipq\m{\Sigma}_\m{A}^{-{1/2}} \fand \m{U}_\m{P}\m{U}_\m{P}^\top\m{P} = \m{P},
\end{align*}
we see that
\begin{align*}
\m{X}_\m{A} - \m{X}_\m{P}\m{W} &= \m{A}\m{U}_\m{A}\Ipq\m{\Sigma}_\m{A}^{-{1/2}} - \m{U}_\m{P}\m{U}_\m{P}^\top\m{A}\m{U}_\m{A}\Ipq\m{\Sigma}_\m{A}^{-{1/2}} + \m{R}_{1,1}
\\&= \m{A}\m{U}_\m{A}\Ipq\m{\Sigma}_\m{A}^{-{1/2}} - \m{P}\m{U}_\m{A}\Ipq\m{\Sigma}_\m{A}^{-{1/2}} - (\m{U}_\m{P}\m{U}_\m{P}^\top\m{A}\m{U}_\m{A}\Ipq\m{\Sigma}_\m{A}^{-{1/2}} - \m{P}\m{U}_\m{A}\Ipq\m{\Sigma}_\m{A}^{-{1/2}}) + \m{R}_1
\\&= (\m{A}-\m{P})\m{U}_\m{A}\Ipq\m{\Sigma}_\m{A}^{-{1/2}} - (\m{U}_\m{P}\m{U}_\m{P}^\top\m{A}\m{U}_\m{A}\Ipq\m{\Sigma}_\m{A}^{-{1/2}} - \m{U}_\m{P}\m{U}_\m{P}^\top\m{P}\m{U}_\m{A}\Ipq\m{\Sigma}_\m{A}^{-{1/2}}) + \m{R}_1
\\&= (\m{A}-\m{P})\m{U}_\m{A}\Ipq\m{\Sigma}_\m{A}^{-{1/2}} - \m{U}_\m{P}\m{U}_\m{P}^\top(\m{A}-\m{P})\m{U}_\m{A}\Ipq\m{\Sigma}_\m{A}^{-{1/2}} + \m{R}_1
\\&= (\m{I} - \m{U}_\m{P}\m{U}_\m{P}^\top)(\m{A}-\m{P})\m{U}_\m{A}\Ipq\m{\Sigma}_\m{A}^{-{1/2}} + \m{R}_1
\\&= (\m{I}-\m{U}_\m{P}\m{U}_\m{P}^\top)(\m{A}-\m{P}) \{ \m{U}_\m{P}\Ipq\m{W} + (\m{U}_\m{A}\Ipq - \m{U}_\m{P}\Ipq\m{W}) \} \m{\Sigma}_\m{A}^{-{1/2}} + \m{R}_1
\\&= (\m{A}-\m{P})\m{U}_\m{P}\Ipq\m{W}\m{\Sigma}_\m{A}^{-{1/2}} + \m{R}_3+ \m{R}_2 + \m{R}_1
\\&= (\m{A}-\m{P})\m{U}_\m{P}\Ipq \{ \m{\Sigma}_\m{P}^{-{1/2}}\m{W} + (\m{W}\m{\Sigma}_\m{A}^{-{1/2}} - \m{\Sigma}_\m{P}^{-{1/2}}\m{W}) \} + \m{R}_3 + \m{R}_2 + \m{R}_1
\\&= (\m{A}-\m{P})\m{U}_\m{P}\Ipq\m{\Sigma}_\m{P}^{-{1/2}}\m{W} + \m{R}_4 + \m{R}_3 + \m{R}_2 + \m{R}_1.
\end{align*}

Applying Proposition~\ref{residual_bounds}, we find that
\begin{align*}
\|\m{X}_\m{A} - \m{X}_\m{P}\m{W}\|_{2 \to \infty} = \|(\m{A}-\m{P})\m{U}_\m{P}\Ipq\m{\Sigma}_\m{P}^{-{1/2}}\|_{2 \to \infty} + \mathrm{O}\left(n^{-\frac{3}{4}}\log^{3\alpha+{3/2}}(n)\right).
\end{align*}

Consequently,
\begin{align*}
\left\|\m{X}_\m{A} - \m{X}_\m{P}\m{W}\right\|_{2 \to \infty} \leq \sigma_d(\m{P})^{-{1/2}}\|(\m{A}-\m{P})\m{U}_\m{P}\|_{2 \to \infty} + \mathrm{O}\left(n^{-\frac{3}{4}}\log^{3\alpha+{3/2}}(n)\right).
\end{align*}

Letting $u$ denote the $j^\mathrm{th}$ column of $\m{U}_\m{P}$, we note that for any $i \in \{1,\ldots,n\}$ we have
\begin{align*}
\{ (\m{A}-\m{P})\m{U}_\m{P} \}_{ij} &= \sum_{k=1}^{n} (\m{A}_{ik}- \m{P}_{ik})u_k
\\&= \sum_{k\neq i} (\m{A}_{ik}- \m{P}_{ik})u_k - \m{P}_{ii}u_i.
\end{align*} 
The latter term is $\mathrm{O}(1)$, while the former is a sum of independent zero-mean random variables satisfying
\begingroup
\addtolength{\jot}{1em}
\begin{align*}
\Prob\left(\left|\sum_{k\neq i} (\m{A}_{ik}- \m{P}_{ik})u_k \right| \geq t\right) &\leq 2\exp\left(\frac{-2t^2}{4\displaystyle\sum_{k\neq i}|u_k|^2}\right)
\\&\leq 2\exp\left(\frac{-t^2}{2}\right)
\end{align*}
\endgroup
by Hoeffding's inequality.  Thus $\{(\m{A}-\m{P})\m{U}_\m{P}\}_{ij} = \mathrm{O}\left(\log^{\alpha+{1/2}}(n)\right)$ almost surely, and hence $\|\{(\m{A}-\m{P})\m{U}_\m{P}\}_i\| = \mathrm{O}\left(\log^{\alpha+{1/2}}(n)\right)$ almost surely by summing over all $j \in \{1,\ldots,d\}$.  Taking the union bound over all $n$ rows then shows that
\begin{align*}
\sigma_d(\m{P})^{-{1/2}}\|(\m{A}-\m{P})\m{U}_\m{P}\|_{2 \to \infty} = \mathrm{O}\left(n^{-{1/2}}\log^{\alpha+{1/2}}(n)\right),
\end{align*}
and consequently that
\begin{align*}
\left\|\m{X}_\m{A}\m{Q}_{n} - \m{X}\right\|_{2 \to \infty} = \mathrm{O}\left(n^{-{1/2}}\log^{3\alpha+{3/2}}(n)\right)
\end{align*}
by multiplying on the right by $\m{Q}_{n} = \m{W}^\top \m{Q}_\m{X}^{-1}$.
\end{proof}

\subsection{Proof of Theorem~\ref{Thm:CLT}}

\begin{proof}
From the proof of Theorem~\ref{Thm:2Inf}, we see that
\begin{align*}
n^{1/2}(\m{X}_\m{A} \m{W}^\top \m{Q}_\m{X}^{-1} - \m{X}) = n^{1/2}(\m{A}-\m{P})\m{U}_\m{P}\Ipq\m{\Sigma}_\m{P}^{-{1/2}}\m{Q}_\m{X}^{-1} + n^{1/2}\m{R},
\end{align*} 
where $\|n^{1/2}\m{R}\|_{2 \to \infty} \to 0$ by Proposition \ref{residual_bounds}.

Recall that the matrix $\m{Q}_\m{X}$ was chosen so that
\begin{align*}
\m{X}\m{Q}_\m{X} = \m{X}_\m{P} = \m{U}_\m{P}\m{\Sigma}_\m{P}^{1/2},
\end{align*}
and so
\begin{align*}
\m{U}_\m{P} \Ipq\m{\Sigma}_\m{P}^{-{1/2}} = \m{U}_\m{P}\m{\Sigma}_\m{P}^{-{1/2}}\Ipq = \m{X}\m{Q}_\m{X}\m{\Sigma}_\m{P}^{-1}\Ipq.
\end{align*} 

Thus
\begin{align*}
n^{1/2}(\m{A}-\m{P})\m{U}_\m{P}\Ipq\m{\Sigma}_\m{P}^{-{1/2}}\m{Q}_\m{X}^{-1} =  n^{1/2}(\m{A}-\m{P})\m{X}\m{Q}_\m{X}\m{\Sigma}_\m{P}^{-1}\Ipq\m{Q}_\m{X}^{-1}.
\end{align*}

Consequently,
\begin{eqnarray*}
n^{1/2}\left\{(\m{A}-\m{P})\m{U}_\m{P}\Ipq\m{\Sigma}_\m{P}^{-{1/2}}\m{Q}_\m{X}^{-1}\right\}_i^\top &=& n^{1/2}(\m{Q}_\m{X}\m{\Sigma}_\m{P}^{-1}\Ipq\m{Q}_\m{X}^{-1})^\top\left\{(\m{A}-\m{P})\m{X}\right\}_i^\top
\\&=& (n\m{Q}_\m{X}\m{\Sigma}_\m{P}^{-1}\Ipq\m{Q}_\m{X}^{-1})^\top\left\{n^{-{1/2}}\sum_{j=1}^n (\m{A}_{ij} - \m{P}_{ij})\m{X}_j\right\}
\\&=& (n\m{Q}_\m{X}\m{\Sigma}_\m{P}^{-1}\Ipq\m{Q}_\m{X}^{-1})^\top\left\{n^{-{1/2}}\sum_{j\neq i} (\m{A}_{ij} - \m{P}_{ij})\m{X}_j\right\}
\\&&- (n\m{Q}_\m{X}\m{\Sigma}_\m{P}^{-1}\Ipq\m{Q}_\m{X}^{-1})^\top (n^{-{1/2}}\m{P}_{ii}\m{X}_i ).
\end{eqnarray*} 

The latter term satisfies
\begin{align*}
\left\|(n\m{Q}_\m{X}\m{\Sigma}_\m{P}^{-1}\Ipq\m{Q}_\m{X}^{-1})^\top (n^{-{1/2}}\m{P}_{ii}\m{X}_i )\right\|_{2 \to \infty} \leq  \left\|(n\m{Q}_\m{X}\m{\Sigma}_\m{P}^{-1}\Ipq\m{Q}_\m{X}^{-1})^\top (n^{-{1/2}}\m{P}_{ii}\m{X}_i)\right\| = \mathrm{O}\left(n^{-{1/2}}\right)
\end{align*}
almost surely.

Conditional on the latent variable $\m{Z}_i = \m{z} \in \Z$ from Definition~\ref{Def:WGRDPG}, we have $\m{P}_{ij} = \phi(\m{z})^\top\Ipq\m{X}_j$, and so
\begin{align*}
n^{-{1/2}}\sum_{j\neq i} (\m{A}_{ij} - \m{P}_{ij})\m{X}_j
\end{align*}
is a scaled sum of $n-1$ independent, zero-mean random variables, each with covariance matrix given by
\begin{align*}
\widehat{\m{\Sigma}}(\m{z}) = \E \left\{ v(\m{z}, \m{Z}) \phi(\m{Z}) \phi(\m{Z})^\top \right\},
\end{align*}
recalling that the function $v(\m{Z}_i, \m{Z}_j)$ gives the variance of $\m{A}_{ij}$. Therefore, by the multivariate central limit theorem,
\begin{align*}
n^{-{1/2}}\sum_{j\neq i} (\m{A}_{ij} - \m{P}_{ij})\m{X}_j \to \mathcal{N}(\mathbf{0},\widehat{\m{\Sigma}}(\m{z})).
\end{align*}

Finally, we consider the terms $(n\m{Q}_\m{X}\m{\Sigma}_\m{P}^{-1}\Ipq\m{Q}_\m{X}^{-1})^\top$.  Using the identities
\begin{align*}
\m{Q}_\m{X} = (\m{X}^\top\m{X})^{-1}\m{X}^\top\m{X}_\m{P}, \quad \fand \m{X}_\m{P}^\top\m{X}_\m{P} = \m{\Sigma}_\m{P},
\end{align*}
we see that
\begin{align*}
\m{Q}_\m{X}\m{\Sigma}_\m{P}^{-1}\Ipq\m{Q}_\m{X}^{-1} &= (\m{X}^\top\m{X})^{-1}\m{X}^\top\m{X}_\m{P}\m{\Sigma}_\m{P}^{-1}\Ipq\m{Q}_\m{X}^{-1}
\\& = (\m{X}^\top\m{X})^{-1}\m{Q}_\m{X}^{-\top}\m{X}_\m{P}^\top\m{X}_\m{P}\m{\Sigma}_\m{P}^{-1}\Ipq\m{Q}_\m{X}^{-1}
\\& = (\m{X}^\top\m{X})^{-1}\m{Q}_\m{X}^{-\top}\Ipq\m{Q}_\m{X}^{-1}
\\& = (\m{X}^\top\m{X})^{-1}\Ipq
\end{align*}
and so
\begin{align*}
(n\m{Q}_\m{X}\m{\Sigma}_\m{P}^{-1}\Ipq\m{Q}_\m{X}^{-1})^\top \to \Ipq\m{\Delta}^{-1}
\end{align*}
almost surely by the law of large numbers.

Combining all this, we find that, conditional on $\m{Z}_i = \m{z}$,
\begin{align*}
n^{-{1/2}}\left(\m{X}_\m{A} \m{W}^\top \m{Q}_\m{X}^{-1} - \m{X}\right)_i^\top \to \mathcal{N}\left(\mathbf{0},\m{\Sigma}(\m{z}) \right)
\end{align*}
almost surely.
\end{proof}

\section{Proofs for size-adjusted Chernoff information}

Before beginning the proofs, it is convenient to write the summations in Corollary~\ref{Cor:WSBM_CLT} as matrix products. The second moment matrix $\m{\Delta}$ can expressed as
\begin{equation*}
	\m{\Delta} = \sum_{k=1}^K \pi_k (\m{X}_\m{B})_k (\m{X}_\m{B})_k^\top
	= \m{X}_{\m{B}}^\top \m{\Pi} \m{X}_{\m{B}},
\end{equation*}
where $\m{\Pi} = \diag(\pi_1, \ldots, \pi_K)$. Similarly, the covariance matrix $\m{\Sigma}_k$ for $k \in \{ 1, \ldots, K\}$, can be expressed as
\begin{equation*}
	\m{\Sigma}_k
	= \Ipq \m{\Delta}^{-1} \left\{ \sum_{\ell=1}^{K} \pi_\ell \m{C}_{k \ell} (\m{X}_{\m{B}})_\ell (\m{X}_{\m{B}})_\ell^\top \right\} \m{\Delta}^{-1} \Ipq
	= \Ipq \m{\Delta}^{-1} \m{X}_\m{B}^\top \m{\Pi} \m{S}_k \m{X}_\m{B} \m{\Delta}^{-1} \Ipq,
\end{equation*}
where $\m{S}_k = \diag(\m{C}_k) \in \R^{K \times K}$. The expression for $\m{\Sigma}_{k \ell}(t)$ has the same form as the above equation, replacing $\m{S}_k$ with its corresponding counterpart $\m{S}_{k \ell}(t) = (1-t) \m{S}_k + t \m{S}_\ell$.

\subsection{Proof of Lemma~\ref{Lem:SACI}}

\begin{proof}
If $\m{D} \in \R^{K \times K}$ is full rank, then for $\m{Y} \in \R^{K \times d}$, the following matrix inequality holds \citep{marshall1990matrix},
\begin{equation*}
    \m{Y} \left( \m{Y}^\top \m{D} \m{Y} \right)^{-1} \m{Y}^\top \preceq \m{D}^{-1},
\end{equation*}
where $\m{M} \succeq \m{0}$ means that $\m{M}$ is a positive semi-definite matrix. However, in the case where $\m{D}$ and $\m{X} \in \R^{K \times K}$ are full rank, then the two sides of this inequality are equal,
\begin{equation*}
    \m{X} \left( \m{X}^\top \m{D} \m{X} \right)^{-1} \m{X}^\top = \m{D}^{-1}.
\end{equation*}
If the block mean matrix $\m{B}$ is full rank, then adjacency spectral embedding $\m{X}_\m{B}$ is also full rank. Since $\m{B} = \m{X}_\m{B} \Ipq \m{X}_\m{B}^\top$, this means $\rank (\m{B}) \le \rank(\m{X}_\m{B})$ implying $\rank(\m{X}_\m{B}) = K$ when $\rank (\m{B}) = K$.

Using this matrix equality and the expression for $\m{\Delta}$, we have
\begin{align*}
    \m{\Sigma}_{k \ell}(t)^{-1}
    &= \left( \Ipq \m{\Delta}^{-1} \m{X}_\m{B}^\top \m{\Pi} \m{S}_{k \ell}(t) \m{X}_\m{B} \m{\Delta}^{-1} \Ipq \right)^{-1} \\
    &= \Ipq \m{X}_\m{B}^\top \m{\Pi} \m{X}_\m{B} \left( \m{X}_\m{B}^\top \m{\Pi} \m{S}_{k \ell}(t) \m{X}_\m{B} \right)^{-1} \m{X}_\m{B}^\top \m{\Pi} \m{X}_\m{B} \Ipq \\
    &= \Ipq \m{X}_\m{B}^\top \m{\Pi} \m{S}_{k \ell}(t)^{-1} \m{X}_\m{B} \Ipq
\end{align*}
Substituting this expression into the objective function in the size-adjusted Chernoff information gives
\begin{align*}
	\lefteqn{\left\{ (\m{X}_\m{B})_k - (\m{X}_\m{B})_\ell \right\}^\top \m{\Sigma}_{k \ell}(t)^{-1} \left\{ (\m{X}_\m{B})_k - (\m{X}_\m{B})_\ell \right\} } \\
	&= (\m{e}_k - \m{e}_\ell)^\top \m{X}_\m{B} \Ipq \m{X}_\m{B}^\top \m{\Pi} \m{S}_{k \ell}(t)^{-1} \m{X}_\m{B} \Ipq \m{X}_\m{B}^\top (\m{e}_k - \m{e}_\ell) \\
	&= (\m{e}_k - \m{e}_\ell)^\top \m{B} \m{\Pi} \m{S}_{k \ell}(t)^{-1} \m{B} (\m{e}_k - \m{e}_\ell),
\end{align*}
using the expression $\m{B} = \m{X}_\m{B} \Ipq \m{X}_\m{B}^\top$.
\end{proof}

If the block mean matrix $\m{B}$ is not full rank, then the matrix inequality effectively passes through the whole argument in the above proof,
\begin{equation*}
    \m{\Sigma}_{k \ell}(t)^{-1} \preceq \Ipq \m{X}_\m{B}^\top \m{\Pi} \m{S}_{k \ell}(t)^{-1} \m{X}_\m{B} \Ipq,
\end{equation*}
which becomes a regular inequality when we compute the quadratic form. We get that, for $\rank(\m{B}) \le K$,
\begin{equation*}
    C \le \min_{k \neq \ell}  \sup_{t \in (0,1)} \left[ \frac{t(1-t)}{2} \left\{ (\m{e}_k - \m{e}_\ell)^\top \m{B} \m{\Pi} \m{S}_{k \ell}(t)^{-1} \m{B} (\m{e}_k - \m{e}_\ell) \right\} \right].
\end{equation*}
with equality when $\m{B}$ is full rank.

\subsection{Proof of Lemma~\ref{Thm:Affine_SBM}}

\begin{proof}
By assumption $\m{A}$ and $\m{A}'$ have full rank block mean matrices, therefore, by Lemma~\ref{Lem:SACI}, both stochastic block models have the same size-adjusted Chernoff information if they have the same value for
\begin{equation*}
    (\m{e}_k - \m{e}_\ell)^\top \m{B} \m{\Pi} \m{S}_{k \ell}(t)^{-1} \m{B} (\m{e}_k - \m{e}_\ell).
\end{equation*}
For an affine entry-wise transformation, the entries of the block mean and variance matrices are given by $\m{B}'_{k \ell} = a \m{B}_{k \ell} + b$ and $\m{C}'_{k \ell} = a^2 \m{C}_{k \ell}$. Therefore,
\begin{align*}
    \m{B}' (\m{e}_k - \m{e}_\ell) &= a \m{B} (\m{e}_k - \m{e}_\ell), \\
    \m{S}'_{k \ell}(t)^{-1} &= a^2 \m{S}_{k \ell}(t)^{-1},
\end{align*}
where $\m{S}'_{k \ell}(t)^{-1}$ is the equivalent version of $\m{S}_{k \ell}(t)^{-1}$ in the entry-wise transformed stochastic block model. The contribution from $a$ cancel and there is no contribution from $b$, meaning size-adjusted Chernoff information is unaffected by affine transformation.
\end{proof}


\subsection{Changes of signature by affine transformation}
We demonstrate the possible effects an affine transformation may have on the signature of a stochastic block model with dimension $d$ and signature $(p,q)$. We focus on affine transformations with $a > 0$ and $b > 0$ although a similar analysis for other affine transformations leads to slight variations of the following results.

Consider a weighted stochastic block model with full rank block mean matrix $\m{B}$ with signature $(p,q)$ and let $\lambda_i (\m{B})$ denote the $i^\mathrm{th}$ smallest eigenvalue of $\m{B}$. The signature of $\m{B}$ implies that $\lambda_1 (\m{B}), \ldots, \lambda_q (\m{B}) < 0$ and $\lambda_{q+1} (\m{B}), \ldots, \lambda_{p+q} (\m{B}) > 0$, since $\m{B}$ is full rank, no eigenvalues are exactly zero.

The block mean matrix of the weighted stochastic block model after affine transformation is given by $\smash{\m{B}' = a \m{B} + b \mathbf{1} \mathbf{1}^\top}$, where $\mathbf{1} \in \R^n$ is the all-one vector. This is a rank-one perturbation of the scaled matrix $a \m{B}$, therefore, by \citet{HJ12}, Corollary 4.3.9, the $i^\mathrm{th}$ eigenvalue of $\m{B}'$ lies between the $i^\mathrm{th}$ and $(i+1)^\mathrm{th}$ eigenvalue of $\m{B}$, that is, for $i = 1, \dots, d-1$,
\begin{equation*}
	\lambda_i(\m{B}) \le \lambda_i(\m{B}') \le \lambda_{i+1}(\m{B}) \fand \lambda_d(\m{B}) \le \lambda_d(\m{B}').
\end{equation*}
Figure~\ref{Fig:EigenWeave} shows an example of this eigenvalue behavior. The only eigenvalue that can change sign after affine transformation is $\lambda_q(\m{B}')$. In the example shown, $\lambda_q(\m{B}') > 0$, meaning the signature of the stochastic block model after affine transformation is $(p+1,q-1)$; if $\lambda_q(\m{B}') < 0$, then the signature would have remained $(p,q)$. The other remaining possibility is that $\lambda_q(\m{B}') = 0$, meaning that $\m{B}'$ is not full rank. In this case, the affine transformation block model would instead need to be embedded into $d-1$ dimensions and would have signature $(p,q-1)$.
\begin{figure}[ht]
	\centering
	\begin{tikzpicture}
		\draw[latex-latex] (-6.5,0) -- (6.5,0);
		\draw[color=black] (0pt,3pt) -- (0pt,-3pt);
		\draw[color=black] (0pt,0pt) -- (0pt,-1.8pt) node[below] {$0$};
		
		\draw[-latex, ultra thick] (-5.0,0) -- (-3.3,0);
		\draw[-latex, ultra thick] (-1.6,0) -- (0.5,0);	
		\draw[-latex, ultra thick] (1.8,0) -- (2.9,0);	
		\draw[-latex, ultra thick] (4.1,0) -- (5.7,0);	
		
		\path [draw=black, fill=white, thick] (-5.0,0.0) circle (2.5pt) node[below] {$\lambda_{q-1}(\m{B})$};
		\path [draw=black, fill=white, thick] (-1.6,0.0) circle (2.5pt) node[below] {$\lambda_{q}(\m{B})$};		
		\path [draw=black, fill=white, thick] (1.8,0.0) circle (2.5pt) node[below] {$\lambda_{q+1}(\m{B})$};
		\path [draw=black, fill=white, thick] (4.1,0.0) circle (2.5pt) node[below] {$\lambda_{q+2}(\m{B})$};
		
		\path [draw=black, fill=black, thick] (-3.2,0.0) circle (2.5pt) node[above] {$\lambda_{q-1}(\m{B}')$};
		\path [draw=black, fill=black, thick] (0.6,0.0) circle (2.5pt) node[above] {$\lambda_{q}(\m{B}')$};		
		\path [draw=black, fill=black, thick] (3.0,0.0) circle (2.5pt) node[above] {$\lambda_{q+1}(\m{B}')$};
		\path [draw=black, fill=black, thick] (5.8,0.0) circle (2.5pt) node[above] {$\lambda_{q+2}(\m{B}')$};
		
		\node at (-6.4,-0.3) {$\ldots$};
		\node at (6.4,-0.3) {$\ldots$};	
		\draw [decorate,decoration={brace,amplitude=5pt,mirror,raise=4ex}]
  (-6.8,0) -- (-1.0,0) node[midway,yshift=-3em]{$q$ negative eigenvalues of $\m{B}$};
		\draw [decorate,decoration={brace,amplitude=5pt,mirror,raise=4ex}]
  (1.0,0) -- (6.8,0) node[midway,yshift=-3em]{$p$ positive eigenvalues of $\m{B}$};		
	\end{tikzpicture}
	\caption{Number line showing the eigenvalues of mean block matrices close to the origin. White nodes represent eigenvalues of $\m{B}$ with full rank $d$ and signature $(p,q)$, black nodes represent eigenvalues of $\smash{\m{B}' = a \m{B} + b \mathbf{1} \mathbf{1}^\top}$ with $a > 0$, $b > 0$, full rank $d$ and signature $(p+1,q-1)$.}
	\label{Fig:EigenWeave}
\end{figure}

\end{document}